%% file: neurips_2026.tex
\documentclass{article}

\PassOptionsToPackage{numbers,compress}{natbib}
\usepackage[preprint]{neurips_2026}

\usepackage[utf8]{inputenc} %
\usepackage[T1]{fontenc}    %
\usepackage{hyperref}       %
\usepackage{url}            %
\usepackage{booktabs}       %
\usepackage{amsfonts}       %
\usepackage{nicefrac}       %
\usepackage{microtype}      %
\usepackage{xcolor}         %
\usepackage{soul}

\usepackage{amsmath,amssymb,amsfonts, bbm}
\usepackage[disable]{todonotes}
\usepackage{titlesec}
\usepackage{wrapfig}
\usepackage{graphicx}
\usepackage{subcaption}
\input{math_commands}

\title{Provable Robustness against Backdoor Attacks \\ via the Primal-Dual Perspective on Differential Privacy}

\author{%
Aman Saxena\textsuperscript{\ensuremath{1,2,3}}, Jan Schuchardt\textsuperscript{\ensuremath{4}}, Yan Scholten\textsuperscript{\ensuremath{1,2,3}}, Stephan Günnemann\textsuperscript{\ensuremath{1,2,3}}\\
\texttt{\{a.saxena, y.scholten, s.guennemann\}@tum.de},  \texttt{jan.a.schuchardt@morganstanley.com}\\
\textsuperscript{\ensuremath{1}}{Department of Computer Science, Technical University of Munich}\\
\textsuperscript{\ensuremath{2}}{Munich Data Science Institute} \textsuperscript{\ensuremath{3}}{MCML}\\
\textsuperscript{\ensuremath{4}}{Machine Learning Research, Morgan Stanley}
}

\begin{document}
\maketitle
\begin{abstract}Randomized smoothing is a powerful tool for certifying robustness to adversarial perturbations, including poisoning attacks via randomized training and evasion attacks via randomized inference. Extending these guarantees to backdoor attacks, where training and test data are jointly perturbed, remains challenging because training- and test-time randomized mechanisms must be analyzed within a single robustness certificate. We address this by connecting randomized smoothing to the dual view of differential privacy through privacy profiles, which provide a numerical procedure for composing heterogeneous mechanisms. The resulting framework enables tight, modular, end-to-end certification of complex, composed mechanisms while leveraging existing analyses of differentially private mechanisms. We instantiate the framework for DP-SGD and Deep~Partition~Aggregation with inference-time smoothing, deriving joint robustness guarantees against both training-time and inference-time attacks. Experiments on MNIST and CIFAR-10 demonstrate the effectiveness of our framework. Overall, we provide a principled and general framework for using composite mechanisms to certify robustness under complex threat models that better capture the capabilities of real-world adversaries.
\end{abstract}

\section{Introduction}
\label{sec:Introduction}
\input{content/Introduction}

\section{Related Work}
\label{sec:related_work}
\input{content/related_work}

\section{Background and Preliminaries}
\label{sec:basics_of_DP}
\input{content/basics_of_DP}

\section{Randomized Smoothing from the Dual Perspective of DP}
\label{sec:Mathematical_Formulatiomn}
\input{content/mathematical_formulation}

\section{Designing Robust Training--Test Procedures}
\label{sec:design_robust_TTP}
\input{content/design_robust_TTP}

\vspace{-0.2cm}
\section{Experimental Evaluation}
\label{sec:experiments}
\input{content/experiments}
\vspace{-0.2cm}

\section{Conclusion}
\label{sec:conclusion}
We establish a unified framework for certifying robustness against poisoning and backdoor attacks by leveraging the primal--dual perspective on differential privacy. Our key insight is that, while the primal (hypothesis-testing) view naturally connects differential privacy to randomized smoothing, the dual view, expressed through privacy profiles and hockey-stick divergences, enables efficient numerical composition of heterogeneous mechanisms. This dual perspective allows us to decompose complex joint training--test procedures into simpler randomized components and compose them to obtain end-to-end robustness guarantees. We demonstrate the practical utility of our framework by deriving joint robustness certificates for novel threat models and tight poisoning certificates for DP-SGD. We further provide an analytical characterization of composing DPA with any base mechanism.

\paragraph{Limitations.} Our approach has several limitations. First, certifying DP-SGD requires sampling many models to estimate the relevant expectations, which introduces a non-trivial computational overhead. However, certification is performed offline and can often be parallelized, making the cost manageable in many practical settings. Second, certification via decomposition scales as \(O(R)\); for example, for the subsampled Gaussian mechanism under \(R\) additions/deletions, one must certify \(R+1\) distinct cases. While this introduces additional complexity, the resulting guarantees remain tractable for moderate threat sizes and provide a level of flexibility not supported by prior approaches. Third, robustness guarantees can come at the cost of utility, a trade-off that is not specific to our approach but rather common in certified robustness and privacy-preserving learning more broadly. Nevertheless, we believe this trade-off can be mitigated through improved training procedures and optimization techniques, making certifiable robustness increasingly practical at scale in the long run.

\clearpage
\section*{Acknowledgments}
We would like to thank Arthur Kosmala for proofreading this manuscript. This work has been funded by the DAAD program Konrad Zuse Schools of Excellence in Artificial Intelligence (sponsored by the Federal Ministry of Education and Research). The authors of this work take full responsibility for its content.

\bibliography{example_paper}
\bibliographystyle{plainnat}

\appendix

\input{content/Appendix}

\end{document}

%% file: math_commands.tex
\usepackage{amsmath,amsfonts,bm}
\usepackage{mathtools}
\usepackage{amsthm, graphicx}
\theoremstyle{plain}
\newtheorem{theorem}{Theorem}[section]

\newtheorem{lemma}[theorem]{Lemma}
\newtheorem{corollary}[theorem]{Corollary}
\theoremstyle{definition}
\newtheorem{definition}[theorem]{Definition}

\theoremstyle{remark}
\newtheorem{remark}[theorem]{Remark}
\usepackage{mathrsfs}
\usepackage{enumitem}

\usepackage{array}
\newcolumntype{L}[1]{>{\raggedright\arraybackslash}m{#1}}
\newcolumntype{R}[1]{>{\raggedright\arraybackslash}m{#1}}

\newcommand{\vect}[1]{\mathbf{#1}}

\newcommand{\classifier}[1][]{f\ifx&#1&\else(#1)\fi}

\newcommand{\logit}[1][]{l\ifx&#1&\else(#1)\fi}

\newcommand{\classpred}[1][]{\eta\ifx&#1&\else(#1)\fi}

\counterwithin*{theorem}{section}
\newtheorem{innercustomgeneric}{\customgenericname}
\providecommand{\customgenericname}{}
\newcommand{\newcustomtheorem}[2]{%
  \newenvironment{#1}[1]
  {%
   \renewcommand\customgenericname{#2}%
   \renewcommand\theinnercustomgeneric{##1}%
   \innercustomgeneric
  }
  {\endinnercustomgeneric}
}

\newcustomtheorem{customthm}{Theorem}
\newcustomtheorem{customlemma}{Lemma}
\newcustomtheorem{customproposition}{Proposition}

\usepackage{amsmath,amssymb,amsfonts}

\def\eqref#1{equation~\ref{#1}}

\def\1{\bm{1}}

\def\vx{{\bm{x}}}
\def\vy{{\bm{y}}}

\DeclareMathAlphabet{\mathsfit}{\encodingdefault}{\sfdefault}{m}{sl}
\SetMathAlphabet{\mathsfit}{bold}{\encodingdefault}{\sfdefault}{bx}{n}

\DeclareMathOperator{\argmax}{\rm{argmax}}

\DeclarePairedDelimiterX{\infdivx}[2]{[}{]}{%
  #1\;\delimsize\|\;#2%
}

%% file: content/Introduction.tex
Adversarial attacks on machine learning systems remain a major concern, particularly with the rapid deployment of large-scale models such as large language models \cite{shayegani2023surveyvulnerabilitieslargelanguage, schwinn2023adversarialattacksdefenseslarge, geisler2025attackinglargelanguagemodels}. These attacks extend beyond test-time evasion \cite{biggio2013evasion,szegedy2014intriguing} to include data poisoning and backdoor attacks \cite{yngqi2018trojan,chen2017targeted}, where adversaries manipulate training data to degrade performance or induce hidden behaviors at inference time, making robustness guarantees under such joint threat settings particularly challenging. While formally verifying neural network robustness remains difficult \cite{katz2017reluplex}, randomized smoothing (RS) has emerged as a powerful approach for certifying robustness against inference-time attacks \cite{lecuyer2019certified,cohen2019randomized, sosnin2025certified}. It has been successfully applied across a range of domains, including image classification \cite{cohen2019randomized, zhang2020blackboxcertificationrandomizedsmoothing}, discrete data such as graphs \cite{lee2020tightcertificatesadversarialrobustness, bojchevski2023efficientrobustnesscertificatesdiscrete, Scholten2022randomized}, and quantum classifiers \cite{Du_2021, saxena2024certifiablyrobustencodingschemes}. However, extending RS to certify robustness under joint training- and inference-time attacks remains largely unexplored, with existing approaches limited to specific threat models and discrete-feature settings \cite{Zhang2022BagFlipAC, Wang2020OnCR}.

To certify robustness to inference-time attacks, RS samples from a distribution over perturbed inputs and bounds how adversarial perturbations affect the induced prediction distribution. Viewing training as a mapping from datasets to test-time predictions, the same idea extends to poisoning robustness by considering distributions induced by perturbations to the training algorithm or  data~\cite{weber2023rab, jia2020intrinsiccertifiedrobustnessbagging, levine2021dpa}. This perspective suggests a possible extension to joint perturbations of training and test data, but it remains unclear how to compose guarantees from training-time and inference-time randomization.
\begin{figure}[t!]
    \centering
    \includegraphics[
        width=0.85\linewidth,
        trim=0cm 4cm 0 4.5cm,
        clip
    ]{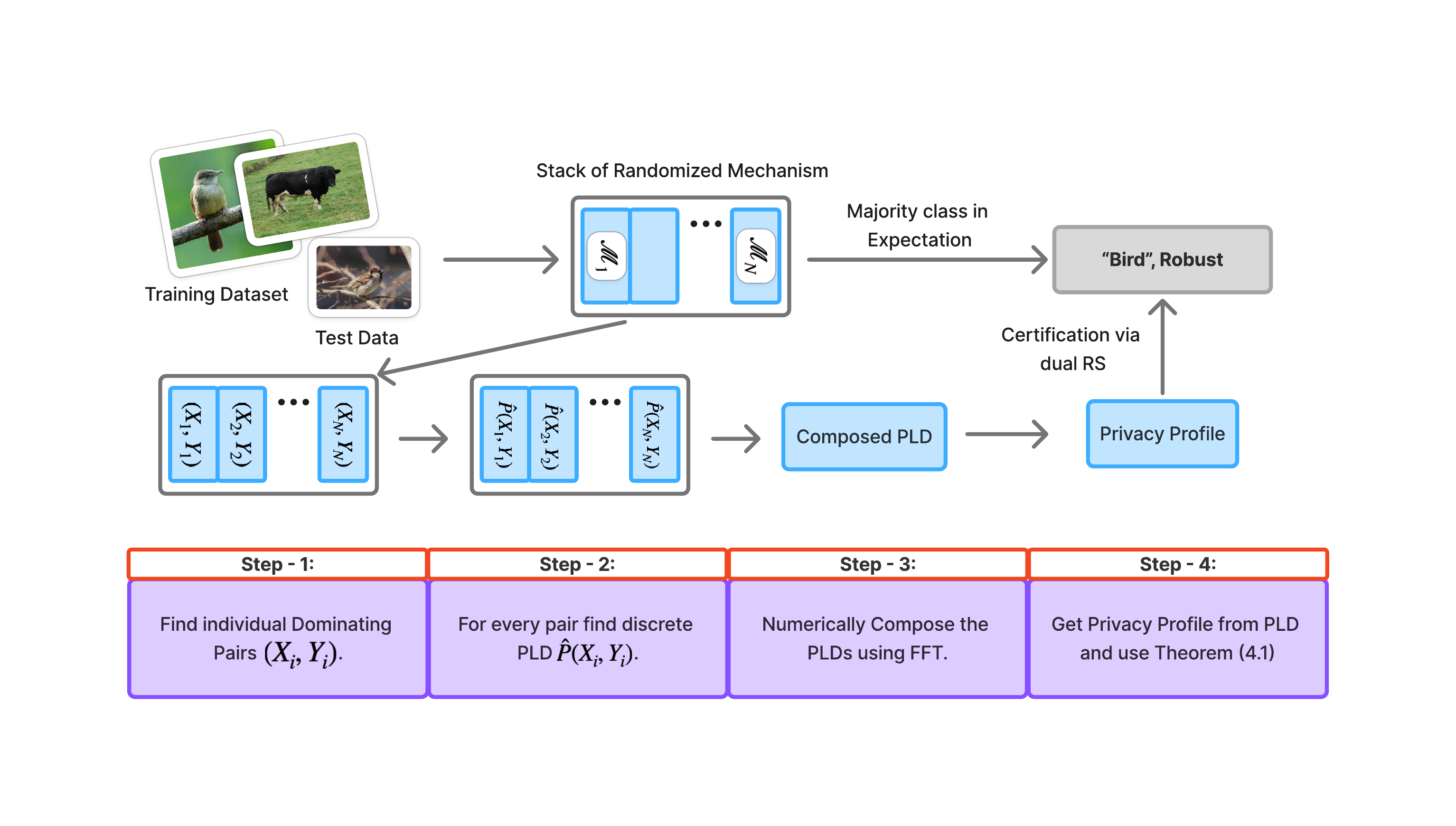}
    \caption{Overview of our framework and novel perspective: The training and classification pipeline can be viewed as a composition of randomized mechanisms, each with its own privacy guarantee. By leveraging the primal-dual equivalence between $f$-DP and privacy profiles, we can efficiently compose these guarantees to derive end-to-end robustness certificates against backdoor attacks.}
    \label{fig:pipeline}
\end{figure}
Compositionality lies at the core of differential privacy (DP)~\cite{Dwork2010BoostingAD, DBLP:journals/corr/abs-2106-08567, dong2019gaussian}, and prior work~\cite{lyu2024adaptive} has leveraged the formalism of $f$-DP to tightly compose the adversarial robustness guarantees for specific randomized smoothing mechanisms such as Gaussian noise. However, while $f$-DP provides a tight theoretical characterization of composition, it does not yield tractable procedures for composition of arbitrary mechanisms, and tractable analysis remains limited to homogeneous settings (e.g., Gaussian or Laplace mechanisms). In contrast, other works connect easily composable notions of DP, such as ADP~\cite{lecuyer2019certified, xie2023unravelingconnectionsprivacycertified} and RDP~\cite{liu2023enhancing}, to randomized smoothing, yielding tractable composition but weaker guarantees than $f$-DP~\cite{riess2026optimalconversionrenyidifferential, dong2019gaussian}. This raises a central question:

{\centering
\emph{How can we tractably compose arbitrary mechanisms while preserving a tight connection to randomized smoothing to certify robustness against complex threat models such as backdoor attacks?}
\par}

In this work, we answer this question by connecting randomized smoothing to the dual perspective of differential privacy, namely privacy profiles. We derive this dual formulation of randomized smoothing by leveraging and extending the primal–dual equivalence between $f$-DP and privacy profiles. This enables us to draw on the extensive literature on tight analyses of commonly used differentially private mechanisms in machine learning~\cite{balle2018amplification, schuchardt2024unifiedmechanismspecificamplificationsubsampling, schuchardt2025privacyamplificationstructuredsubsampling}, together with numerical composition techniques (e.g., via PLDs~\cite{doroshenko2022connectdotstighterdiscrete, Meiser2018TightOB}), thereby operationalizing the connection between $f$-DP and RS.

\textbf{Contributions}
We introduce a general framework for provable robustness that enables a modular design of randomized smoothing-based classifiers composed of simpler mechanisms. We provide an overview of our framework in \autoref{fig:pipeline}. Our approach tracks the privacy of individual components under a given threat model and composes these guarantees to derive robustness certificates via the dual formulation of randomized smoothing. This framework directly enables the evaluation of backdoor robustness under joint threat models. We demonstrate certified backdoor robustness on MNIST and CIFAR-10 for compositions of inference-time Gaussian noise with training-time randomization, including Deep Partition Aggregation (DPA)~\cite{levine2021dpa} and DP-SGD. This modularity further allows us to analyze complex training mechanisms in isolation. In particular, we derive provably tight randomized smoothing guarantees for DP-SGD~\cite{abadi2016deep} under additions and removals of training data, and empirically demonstrate improvements over \citet{liu2023enhancing}.%

%% file: content/related_work.tex
\paragraph{Certification Against Backdoor Attacks}
\todo{Add the only deterministic certification.}Randomized smoothing has become a standard tool for certifiable robustness against evasion attacks~\citep{cohen2019randomized, lecuyer2019certified, li2023sok}. Building on this foundation, several works have explored joint certification using RS~\citep{Zhang2022BagFlipAC, Wang2020OnCR, Zhang2023PECANAD}. Other works adopt notions of backdoor robustness that do not capture joint training-time and inference-time perturbations~\citep{weber2023rab, Qiao2025CertSSBDCB, jia2021certifiedrobustnessnearestneighbors}. In addition, \citep{sun-etal-2024-crowd} proposes an RS-based approach tailored to language-based tasks, while \citet{lorenz2024fullcertdeterministicendtoendcertification} relies on bound-propagation techniques applied to the unrolled training procedure. 

Overall, these approaches are limited to simple training procedures, specific model classes or tasks, and narrowly defined threat models that are easy to analyze and compose. In contrast, we develop a general framework that enables the use of tools from differential privacy to analyze and compose more general randomized mechanisms. As a result, randomized smoothing can be applied to general training and inference procedures under broad threat models.

\textbf{Certification Against Poisoning Attacks (Without DP)}
Beyond joint certification, a large body of work studies poisoning robustness in isolation. Most RS-based poisoning certification methods rely on aggregation-based techniques, typically variants of DPA~\cite{levine2021dpa, Hammoudeh_2024, wang2022improvedcertifieddefensesdata, jia2020intrinsiccertifiedrobustnessbagging, rezaei2023runoffelectionimprovedprovable, Scholten2025provably, gosch2026cert}. Another line of work~\cite{rosenfeld2020certifiedrobustnesslabelflippingattacks} uses randomized smoothing with a task-specific smoothing mechanism for label-flipping attacks. While effective in specific settings, these approaches are tightly coupled to particular randomized training mechanisms and are therefore tied to specific threat models. In contrast, our framework supports a broad class of randomized training procedures that admit a privacy-profile characterization under a given threat model, along with their principled composition. 

\textbf{Certification Against Poisoning Attacks (With DP)}
Several works exploit the connection between differential privacy and robustness certification \cite{xie2023unravelingconnectionsprivacycertified, borgnia2021dpinstahideprovablydefusingpoisoning, naseri2022localcentraldifferentialprivacy, liu2023enhancing}. \citet{du2019robustanomalydetectionbackdoor} uses differential privacy for adversarial detection. This connection has been similarly studied in the context of evasion robustness \cite{lecuyer2019certified, lee2019tight, lyu2024adaptive}. Existing approaches rely on a coarse characterization of differential privacy via ADP, or Rényi Differential Privacy~\cite{jia2020intrinsiccertifiedrobustnessbagging}, leading to suboptimal robustness bounds. In contrast, our work uses a privacy-profile–based formulation of differential privacy, which is equivalent to the hypothesis-testing formulation underlying randomized smoothing~\cite{cohen2019randomized, lyu2024adaptive}. %

\textbf{Robustness and Differential Privacy} Adaptive randomized smoothing~\cite{lyu2024adaptive} connects $f$-DP to robustness certification. However, this perspective does not provide an explicit composition rule, except for homogeneous composition of basic mechanisms such as the addition of Gaussian and Laplace noise. We use the primal-dual perspective of DP~\cite{dong2019gaussian} to convert the $f$-DP interpretation of randomized smoothing into a dual (privacy profile) formulation. This enables tractable composition of randomized mechanisms and allows robustness certification to directly leverage privacy analyses from the DP literature, such as tight guarantees for subsampled Gaussian mechanisms~\citep{schuchardt2024unifiedmechanismspecificamplificationsubsampling}. 

The perspective that privacy profiles provide an optimal characterization of robustness has appeared implicitly in prior work~\cite{dvijotham2020framework, zhang2020blackboxcertificationrandomizedsmoothing}, which derive RS guarantees for $\ell_p$-norm evasion certification, that resembles our dual-to-primal conversion in~\autoref{thm:primal_dual_asymmetric}. However, these works do not establish any connection to differential privacy, and therefore cannot leverage existing tools and analyses from the DP literature. This work expands upon the use of privacy accounting proposed in Dissertation \cite{schuchardt2026probabilistic}.

\todo{Add Non-RS papers, and see if CRFL makes sense}

%% file: content/basics_of_DP.tex
\textbf{Joint Training-Inference Procedure} Let each training/test sample $\vx$ come from a space $\mathbb X$ and the associated label $y $ from a set $ \mathbb Y := \{1, \ldots, C\}$. We represent a classifier as a function $f : \mathbb X \to \mathbb Y$, and denote the space of all such classifiers by $\mathbb Y^{\mathbb X}$. For $i \in \mathbb{N}$, the set $T_i := (\mathbb X \times \mathbb Y)^i$ denotes the space of training datasets of size $i$. %
A training algorithm $\mathcal A$ maps a training dataset of arbitrary size to a classifier, i.e., $\mathcal A: \cup_{i=1}^{\infty} T_i \to \mathbb{Y}^{\mathbb{X}}$. Finally, the joint training and inference procedure is a function $\mathcal J: \cup_{i=1}^{\infty} T_i \times \mathbb{X} \mapsto \mathbb{Y}$ that takes a training dataset $\vect{X}_{\text{train}}$, produces a classifier $f = \mathcal A(\vect{X}_{\text{train}})$, and outputs the prediction for a test input $\vx_{\text{test}} \in \mathbb X$, i.e., $\mathcal J: (\vect{X}_{\text{train}}, \vx_{\text{test}}) \mapsto \mathcal{A}(\vect{X}_{\text{train}})(\vx_{\text{test}})$.

\textbf{Backdoor Threat Model}
Let the input space $\mathbb X$ be equipped with a distance $d_{\text{input}}$ and the space of training datasets $\cup_{i=1}^{\infty} T_i$ with a distance $d_{\text{train}}$. Backdoor robustness requires that small perturbations to both the training data and the test input do not change the prediction of the joint training–inference procedure. Specifically, for $(\vect{X}_{\text{train}}, \vx_{\text{test}})$ and all $(\vect{X'}_{\text{train}}, \vx'_{\text{test}})$ such that $d_{\text{train}}(\vect{X'}_{\text{train}}, \vect{X}_{\text{train}}) \leq R$ and $d_{\text{input}}(\vx'_{\text{test}}, \vx_{\text{test}}) \leq \rho$, we call a joint training-inference procedure $\mathcal J$, $(R,\rho)$-robust w.r.t. $(d_{\text{train}}, d_{\text{input}})$ at $(\vect{X}_{\text{train}}, \vx_{\text{test}})$ if $\mathcal J(\vect{X'}_{\text{train}}, \vx'_{\text{test}}) = \mathcal J(\vect{X}_{\text{train}}, \vx_{\text{test}})$. We use “backdoor robustness” and “joint training–inference robustness” interchangeably, occasionally abbreviated as “joint robustness”.

\textbf{Randomized Mechanisms}
 The notion of a randomized mechanism provides a common abstraction for randomized smoothing and differential privacy. A \emph{randomized mechanism} $\mathcal{M}$ is a random function that maps an input $X$ to a distribution over a space $\mathcal Z$, denoted as $\mathcal M(\cdot \mid X)$. Examples include: (i) adding Gaussian noise, mapping $\vx \mapsto \mathcal N(\vx, \sigma^2 I)$; (ii) stochastic training procedures such as DP-SGD, which induce a distribution over model parameters given a dataset; and (iii) randomized subsampling, which defines a distribution over datasets.
 
\textbf{Randomized Smoothing: $C$-class Smooth Classifier} Randomized smoothing derives worst-case, typically black-box, robustness guarantees by applying a classifier to samples from a randomized mechanism and aggregating the resulting predictions. Let $\mathcal M$ be a randomized mechanism mapping an input $X$ to an intermediate space $\mathcal Z$. Further, let $\{\phi_i\}_{i=1}^C$ be a collection of hypothesis tests\footnote{A \emph{hypothesis test} is a randomized mechanism $\phi : \mathcal Z \to [0,1]$, where $\phi(z)$ denotes the probability of predicting $1$.} $\phi_i : \mathcal Z \to (\{0,1\}),
\quad i \in \{1,\dots,C\}$ with $\sum_{i} \phi_i = 1$.
The corresponding smoothed class probabilities are defined as $p_i
\;:=\;
\mathbb E_{z \sim \mathcal M(\cdot \mid X)}
\bigl[\phi_i(z)\bigr],
\quad i \in \{1,\dots,C\}$.
The induced smoothed classifier predicts the label $\argmax_{i \in \{1,\dots,C\}} \; p_i$. Here, $X$ may represent either a test input, a training dataset, or a training–test pair, depending on the context. For instance, $\mathcal M$ adds Gaussian noise to the test input, while a neural network represents the collection of hypothesis tests $\{\mathcal{\phi}_i\}_{i=1}^{C}$.

\textbf{Differential Privacy}
Differential privacy~\cite{dwork2006calibrating,dwork2014algorithmic} studies the closeness of the output distributions $\mathcal M(\cdot \mid X)$ and $\mathcal M(\cdot \mid X')$ induced by a randomized mechanism $\mathcal M$ on neighboring inputs $X \approx X'$, where $\approx$ denotes a relation on $\mathcal X$. %
This closeness is typically quantified either through hypothesis–testing–based notions~\cite{wasserman2010statistical,kairouz2015composition,dong2019gaussian} or through divergence-based measures~\cite{barthe2013beyond,balle2018amplification}. The two equivalent characterizations, $f$-DP and privacy profiles, are respectively grounded in these perspectives. We briefly discuss this equivalence in Section~\ref{sec:Mathematical_Formulatiomn}, with details in Appendix~\ref{Appendix:sec:fdp_privacy_profile_equivalence}.

\textbf{Privacy Profiles: The Dual Perspective of DP}
Classically, differential privacy is defined via a pair $(\varepsilon,\delta)$, known as \emph{$(\varepsilon,\delta)$-differential privacy}~\cite{dwork2006data}, also referred to as approximate differential privacy (ADP). A randomized mechanism $\mathcal M$ is said to be $(\varepsilon,\delta)$-DP if, for all neighboring inputs $X \approx X'$ and all events $B$, $\mathcal M(B \mid X) \le e^{\varepsilon}\mathcal M(B \mid X') + \delta$. While this definition captures privacy guarantees at a fixed pair $(\varepsilon,\delta)$, the trade-off between $(\varepsilon,\delta)$ values is needed to tightly characterize composition. Privacy profiles address this limitation by characterizing, for each $\varepsilon$, the smallest $\delta$ such that the mechanism satisfies $(\varepsilon,\delta)$-DP. This refinement can be expressed via hockey-stick divergences~\cite{barthe2013beyond}, defined as %
{$D^{\mathrm{HS}}_{\alpha}(P \,\|\, Q) = \int (p(y) - \alpha q(y))_+ \, dy$}, for the respective densities $p$, and $q$. 
The \emph{privacy profile} of $\mathcal M$ w.r.t.\ $\approx$ is a function $\delta : \mathbb R \to [0,1]$ defined by
\[
\delta(\varepsilon)
\;=\;
\sup_{X \approx X'}
D^{\mathrm{HS}}_{e^{\varepsilon}}\!\left(\mathcal M(X) \,\middle\|\, \mathcal M(X')\right).
\]

\textbf{Dominating Pairs}
Given a mechanism $\mathcal M$ with privacy profile $\delta(\varepsilon)$ for the neighboring relation $\approx$, any pair of distributions $(P,Q)$ satisfying $\delta(\varepsilon) \leq D^{\mathrm{HS}}_{e^{\varepsilon}}\!\left( P \,\middle\|\, Q\right)
\quad \text{for every } \varepsilon \in \mathbb R$ is called a \emph{dominating pair} of $(\mathcal M,\approx)$. If equality holds for every $\varepsilon \in \mathbb R$, then $(P,Q)$ is called a \emph{tight dominating pair}.

\textbf{$f-$DP: The Primal Perspective of DP}
The notion of $f-$DP exploits a hypothesis-testing perspective to compare the distributions induced by neighboring inputs. Distinguishing two distributions $P$ and $Q$ is characterized by the trade-off between Type~I and Type~II errors, formalized by the tradeoff function $\Lambda(P \,\|\, Q)$~\cite{dong2019gaussian}. It evaluates the minimal Type~II error at a given Type~I error level, over all hypothesis tests $\phi \in \mathcal H$:
\[
\Lambda(P \,\|\, Q)(\alpha)
\;=\;
\inf_{\phi \in \mathcal H} \; \mathbb{E}_{z \sim Q}[1 - h(z)]
\quad \text{s.t.} \quad
\mathbb{E}_{z \sim P}[h(z)] = \alpha .
\]
A function $f : [0,1] \to [0,1]$ is a tradeoff function iff it is convex, continuous, non-increasing, and $f(\alpha) \le 1 - \alpha$ for all $\alpha$. 
Given a tradeoff function $f$, a mechanism $\mathcal M$ is said to satisfy \emph{$f$-differential privacy} if, for all $X \approx X'$, $\;\Lambda\!\left(\mathcal M(X) \,\middle\|\, \mathcal M(X')\right) \;\ge\; f $.

\textbf{Randomized Smoothing from the Primal Perspective of DP}
  Certification of the $C-$class smooth classifier as defined above reduces to bounding worst-case changes in the class probabilities $p_i$ under perturbations $X'$ satisfying $d(X, X') \leq \delta$ for an appropriate distance $d$. As detailed in Appendix~\ref{Appendix:sec:RS}, this leads to the following optimization problem, which closely resembles the $f$-DP formulation:
\begin{equation}
    \label{eq:MainOPT}
    \begin{aligned}
        \text{Opt}(\alpha) := \inf_{\phi} \inf_{d(X,X') \le \delta}
        \; \mathbb{E}_{Z \sim \mathcal{M}(\cdot \mid X')}[\phi(Z)] \quad \text{s.t.} \quad
        \mathbb{E}_{Z \sim \mathcal{M}(\cdot \mid X)}[\phi(Z)] = \alpha.
    \end{aligned}
\end{equation}
The connection between $f-$DP and the above optimization problem is formalized in \citet{lyu2024adaptive}, and we restate it in Lemma~\ref{lemma:fdp_reformulation}.
\begin{lemma}
\label{lemma:fdp_reformulation}
Let $\approx_\delta$ be the relation defined by $X \approx_\delta X'$ if and only if $d(X,X') \le \delta$. If the mechanism $\mathcal M$ satisfies $f$-differential privacy with respect to $\approx_\delta$, then $\text{Opt}(\alpha) \ge f(1-\alpha)$.
\end{lemma}
However, the primal perspective (Lemma~\ref{lemma:fdp_reformulation}) mainly provides a reformulation of the original optimization problem. To obtain the trade-off function $f$, one must still solve Equation~\ref{eq:MainOPT}, e.g., via the Neyman--Pearson Lemma~\ref{lem:NP}.In section~\ref{sec:Mathematical_Formulatiomn}, we discuss further limitations of this perspective and connect randomized smoothing to the dual perspective of DP via privacy profiles.

%% file: content/mathematical_formulation.tex
Our goal is to derive robustness guarantees under the joint training--inference threat model introduced in Section~\ref{sec:basics_of_DP}. We randomize the training and inference procedures separately, yielding the composed mechanism $\mathcal M_I \circ \mathcal M_T$, which samples a model $m \sim \mathcal M_T(\cdot \mid X_{\text{train}})$ followed by sampling from $\mathcal M_I(\cdot \mid m, x_{\text{test}})$. The composed mechanism finally defines a robust classifier on the joint training -- test space (Definition~\ref{def:Joint-training-procedure}). However, as discussed in Appendix~\ref{appendix_subsec:limitaions_of_f-DP}, the primal perspective handles composition tightly at the theoretical level; it does not provide a tractable way to compose general compositions. To address this, we introduce the \emph{dual} formulation of randomized smoothing.

\paragraph{Basic Setup}
Denote $\approx_\rho$ as the neighboring relation for the threat model $\{{X'} : \;\; d( {X'}, X) \leq \Delta \}$, i.e., we represent ${X'}$ as an adversary within the radius $\Delta$ of $X$ by ${X'} \approx_\Delta X$. For all such $ {X'}$, we want to show that a $C-$class smooth classifier $g(\cdot)$ induced by $(\mathcal M, \{\phi_i\}_{i=1}^{C})$ (as defined in Section~\ref{sec:basics_of_DP}) satisfies $g({X'}) = g(X)$. Suppose that this relation decomposes into a collection of relations $\{\approx{\rho_i}\}_{i=1}^{N}$, i.e., $X \approx_{\rho} X^{'}$ iff there is some $i$ such that $X \approx_{\rho_{i}} X^{'}$, and that the randomized mechanism $\mathcal M$ admits privacy profiles $\delta_i(\varepsilon)$ with respect to $\approx_{\rho_i}$. The following theorem establishes robustness guarantees under the threat model induced by $\{\approx_{\rho_i}\}_{i=1}^N$ via the privacy profiles $\{\delta_i(\varepsilon)\}_{i=1}^N$.
\begin{theorem}
\label{thm:main_theorem}
For the unperturbed input $X$, let $c_1$ and $c_2$ denote the majority and second-majority classes, respectively, and let $p_1, p_2 \in [0,1]$ satisfy $\mathbb{E}_{\mathcal M (\vect X)}\bigl[\phi_{c_1}(z)\bigr] \;\ge\; p_1 \;>\; p_2 \;\ge\; \mathbb{E}_{z \sim \mathcal M(\vect X)}\bigl[\phi_{c_2}(z)\bigr]$. Then the classifier $g(\cdot)$ is $\rho$-robust at $X$, i.e., for all $X^{'} \approx_\rho X$, $g(X^{'}) = g(X)$ if
\begin{equation*}
\min_{i \in \{1, ..., N\}} \bigl[\max_{\varepsilon \in \mathbb R} e^{-\varepsilon}\bigl(p_1 - \delta_i(\varepsilon)\bigr) + \max_{\varepsilon \in \mathbb R} e^{-\varepsilon}\bigl(1 - p_2 - \delta_i(\varepsilon)\bigr)\bigr]
\;>\;
1   
\end{equation*}
\end{theorem}

Intuitively, Theorem~\ref{thm:main_theorem} certifies robustness by checking each decomposed neighboring relation $\approx_i$ separately. This relies on the connection between $f-$DP and randomized smoothing (Lemma~\ref{lemma:fdp_reformulation}) together with the primal--dual equivalence of DP (Theorem~\ref{thm:primal_dual_asymmetric}). We defer the proof to Appendix~\ref{APP_Sec: Main_robustness_results}.   
\paragraph{Connection to $f-$DP}
For brevity, we interpret the above result in the binary classification setting. Let $c_1$ be the majority class for the unperturbed sample $X$, and let $p_1$ denote the lower bound on the smooth probability $\mathbb{E}_{\mathcal M (X)}\bigl[\phi_{c_1}(z)\bigr]$. Deriving robustness guarantees then amounts to solving the optimization problem in Equation~\ref{eq:MainOPT} and checking whether $\mathrm{Opt}(p_1) > 0.5$. This can be rewritten as the infimum of the tradeoff functions $\Lambda(\mathcal M(X)\,\|\,\mathcal M(X'))(1-p_1)$ over perturbed inputs $X'$ such that $X' \approx_\rho X$. \citet{lyu2024adaptive} lower bound this infimum using the $f-$DP characterization of $\mathcal M$ with respect to the relation $\approx_{\rho}$. We strengthen this bound by decomposing the space of perturbed inputs using the relations $\{\approx_{\rho_i}\}_{i=1}^{N}$, when the corresponding $\delta_i(\cdot)$ are tractable to obtain. To utilize these privacy profiles, we leverage the primal -- dual connection of DP. 

\paragraph{Exploiting the Primal--Dual Equivalence} 
\citet{dong2019gaussian} establish the equivalence between $f-$DP and privacy profiles for symmetric neighboring relations. We extend this equivalence to general, potentially asymmetric, relations. This is useful even when $\approx_\rho$ is symmetric, since decomposing it into asymmetric sub-relations can yield tighter guarantees. For instance, as discussed later, the symmetric relation of $R$ additions/deletions on training datasets can be decomposed into asymmetric relations with exactly $r_+$ additions and $r_-$ deletions, yielding provably stronger guarantees (Figure~\ref{fig:convexification_tf}). 
Essentially, we show that the privacy profile $\delta(\cdot)$ of a mechanism $\mathcal M$ with respect to a relation $\approx$ is equivalent to the convex biconjugate of the infimum over tradeoff functions $\Lambda(\mathcal M( X),|,\mathcal M( X'))$. We formalize this via the Optimal Tradeoff Function (Definition~\ref{def:optimal_tradeoff_function}, Theorem~\ref{th:optimal_tradeoff_function}) and state the equivalence in Theorem~\ref{thm:primal_dual_asymmetric}, with details and proofs deferred to Appendix~\ref{Appendix:sec:fdp_privacy_profile_equivalence}.
\begin{definition}(Optimal Tradeoff Function)
    \label{def:optimal_tradeoff_function}
    Given a mechanism $\mathcal M$, and neighboring relation $\approx$, we say that a tradeoff function $f$ is the optimal tradeoff function if $\mathcal M$ is $f-$DP w.r.t., $\approx$, and for any tradeoff function $f{'}$ if $\mathcal M$ is $f^{'}-$DP w.r.t., $\approx$, then $f^{'} \leq f$.    
\end{definition}
\begin{theorem}
\label{thm:primal_dual_asymmetric}
Let $\mathcal M$ be a randomized mechanism with optimal tradeoff function $f$ and privacy profile $\delta$ for the neighboring relation $\approx$. The following equivalent conversions hold:
\begin{align*}
\text{(Primal to Dual)} \quad 
\delta(\varepsilon) 
&= 1 + (f^{-1})^*(-e^{\varepsilon}), \\
\text{(Dual to Primal)} \quad 
f(\alpha) 
&= \sup_{\varepsilon \in \mathbb{R}} e^{-\varepsilon}(1 - \delta(\varepsilon) - \alpha).
\end{align*}
Here, $f^{-1}$ represents the left continuous inverse of a non-decreasing function, and $f^{*}$ represents the convex conjugate of $f$. 
\end{theorem}

\paragraph{Why decompose $\approx_\rho$ further?} The \emph{Dual to Primal} part of Theorem~\ref{thm:main_theorem} suggests that taking the maximum over $\delta_i(\cdot)$ and then converting to $f$ yields a lower bound on first converting each $\delta_i$ to $f_i$ and then taking the minimum. As discussed above, we need the latter for robustness certification. We illustrate this using a subsampled Gaussian mechanism under $100$ changes ($\approx_{100}$) in Figure~\ref{fig:convexification_tf}. The pointwise minimum over the decomposed tradeoff functions lies strictly above the tradeoff function for the non-decomposed relation $\approx_{100}$.

\paragraph{Advantages of the Dual Perspective} 
\begin{wrapfigure}{r}
{0.45\textwidth}
    \centering
 \includegraphics[width=0.43\textwidth]{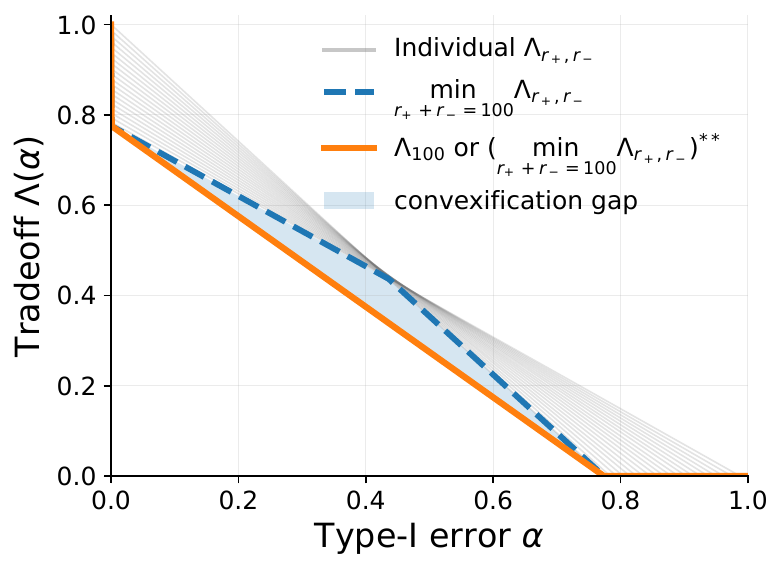}
    \caption{\small Minimum over decomposed tradeoff functions vs.\ the tradeoff function for the non-decomposed relation($\approx_{100}$). $\Lambda_{r_+, r_-}$: subsampled Gaussian ($\gamma=0.00256$, $\sigma=0.05$).}  
    \label{fig:convexification_tf}
\end{wrapfigure}
\emph{Numerical composition and leveraging the DP literature:} In contrast to the primal perspective, representing mechanisms via their privacy profiles enables efficient, algorithm-agnostic numerical accounting, e.g., through convolution of privacy-loss distributions~\citep{sommer2018privacy,koskela2020fft}, thereby yielding tight privacy guarantees for complex, heterogeneous mechanisms. Moreover, the dual perspective allows us to directly leverage existing analyses of mechanisms relevant to machine learning. For example, \citet{schuchardt2024unifiedmechanismspecificamplificationsubsampling} derive tight dominating pairs for the subsampled Gaussian mechanism underlying DP-SGD under additions/removals.
\emph{Tightness.} For a collection of tradeoff functions $\Lambda_i$ corresponding to $\delta_i(\cdot)$, any tradeoff function $f$ such that $\mathcal M$ is $f-$DP for $\approx_\rho$ satisfies $f \le \min_i \Lambda_i$, with strict inequality in general (Figure~\ref{fig:convexification_tf}). Therefore, the primal perspective of~\citet{lyu2024adaptive} yields provably weaker guarantees than decomposition-based analysis in Theorem~\ref{thm:main_theorem}. Moreover, if each $\delta_i(\cdot)$ is attained by a pair of datasets, the resulting guarantees are black-box tight (Lemma~\ref{app:lemma:tightness})\todo{Write this as Theorem}.

These properties enable a modular design of provably robust classifiers from simpler components without losing the guarantees of RS. The next section describes how to operationalize this framework and instantiate it in different settings.

%% file: content/design_robust_TTP.tex
We adopt a modular pipeline separating \emph{threat modeling}, \emph{mechanism design}, and \emph{privacy accounting}. Given a threat model, we (i) represent the mechanism as a stack of simpler randomized components with tractable dominating pairs with respect to the neighboring relation induced by the threat model, (ii) discretize and (iii) compose the dominating pairs to obtain a discrete dominating pair for the overall mechanism, and (iv) derive privacy profiles to certify robustness via Theorem~\ref{thm:main_theorem} (see \autoref{fig:pipeline}). As discussed in Section~4, we decompose the neighboring relation into subsets when needed for tighter analysis. Numerical privacy accounting, i.e., step~(ii), has been studied extensively~\cite{choquettechoo2024privacyamplificationmatrixmechanisms, feldman2026efficientprivacylossaccounting, ghazi2022fasterprivacyaccountingevolving, doroshenko2022connectdotstighterdiscrete, koskela2021tightdifferentialprivacydiscretevalued, koskela2021computingdifferentialprivacyguarantees, gopi2021numericalcompositiondifferentialprivacy, DBLP:journals/corr/abs-2106-08567, koskela2020fft, sommer2018privacy, Meiser2018TightOB}. In our implementation, we use the method of~\citet{doroshenko2022connectdotstighterdiscrete}, as implemented in the Google Differential Privacy library~\cite{google_dp_accounting}, to obtain an upper bound on the composed privacy profile. We illustrate these steps through the following examples.

\paragraph{Example 1: Poisoning Certification using DP-SGD}
\emph{(a) Threat model and mechanism:}
To certify against poisoning attacks where an adversary can add or remove up to $R$ training examples from $\vect X_{\text{train}}$, we employ DP-SGD training. Each iteration of DP-SGD involves applying Poisson subsampling with rate $\gamma$ to get a minibatch, i.e., including each training record i.i.d. with probability $\gamma$. This is followed by adding Gaussian noise $z\sim \mathcal N(0, C^2\sigma^2)$ to the gradients evaluated on the minibatch and clipped in $\ell_2$ norm to $C$. Let $\mathcal M_i$ denote the randomized mechanism at iteration $i$. After $I$ iterations, the composed mechanism $\mathcal M_{\text{train}} := \mathcal M_I \circ \cdots \circ \mathcal M_1$ induces a distribution $\mathcal M_{\text{train}}(\vect X_{\text{train}})$ over learned parameters $w$. \emph{(b) Decomposition of the neighboring relation:}
We decompose the relation $\approx_R$ into $\{\approx_{r_+, r_-}\}_{r_+ + r_- = R}$ corresponding to $r_+$ additions and $r_-$ removals. This decomposition is advantageous as it admits tractable~\cite[Theorem~3.8]{schuchardt2024unifiedmechanismspecificamplificationsubsampling} and tight dominating pairs in the black-box sense~\cite[Theorem~M.4]{schuchardt2024unifiedmechanismspecificamplificationsubsampling}. Specifically, latter implies that \autoref{thm:main_theorem} achieves a tight certificate for DP-SGD in the black-box sense (see Lemma~\ref{app:lemma:tightness})\emph{(c) Dominating pairs:}
For each mechanism $\mathcal M_i$, the dominating pair $(P_{r_+,r_-}, Q_{r_+, r_-})$ is given by
\[
P_{r_+,r_-} = \sum_{i=0}^{r_-} \binom{r_-}{i}\gamma^i(1-\gamma)^{r_- - i}\mathcal N(i,\sigma^2),
\quad
Q_{r_+,r_-} = \sum_{j=0}^{r_+} \binom{r_+}{j}\gamma^j(1-\gamma)^{r_+ - j}\mathcal N(-j,\sigma^2).
\]
\paragraph{Example 2: Joint Certification using DP-SGD and Inference-time Gaussian Noise}
\emph{(a) Threat model and mechanism:} We consider a joint threat model where an adversary can add or remove up to $R$ training examples and perturb the test input by at most $\delta$ in $\ell_2$ norm. Building on the DP-SGD training mechanism $\mathcal M_{\text{train}}$ defined above, we apply an additional Gaussian mechanism at inference time, $\mathcal G_\sigma(\vect x_{\text{test}}) := \mathcal N(\vect x_{\text{test}}, \sigma^2 \mathbb I)$, resulting in $\mathcal M(\vect X_{\text{train}}, \vect x_{\text{test}})
=
\bigl(
w \sim \mathcal M_{\text{train}}(\vect X_{\text{train}}),
\;
\tilde{\vect x} \sim \mathcal G_\sigma(\vect x_{\text{test}})
\bigr)$. \emph{(b) Decomposition of the neighboring relation:} We extend the decomposition to the joint relation $\approx_{R,\delta}$ as $\{\approx_{r_+, r_-, \delta}\}_{r_+ + r_- = R}$, where the training component is handled as before, while the $\ell_2$ perturbation is handled directly without further decomposition, as the worst-case HS divergence is attained by a pair of inputs at distance $\delta$. \emph{(c) Dominating pairs:} The training component uses $(P_{r_+,r_-}, Q_{r_+,r_-})$ as defined above, while the Gaussian mechanism admits the dominating pair $(\mathcal N(0,\sigma^2), \mathcal N(\delta,\sigma^2))$. These are composed to obtain a dominating pair for the joint mechanism.
\paragraph{Example 3: Composing DPA with an arbitrary Mechanism}
We reinterpret DPA~\cite{levine2021dpa} as a randomized training mechanism, derive its privacy profile, and an analytical tradeoff function for a neighboring relation induced by up to $R$ changes to the dataset (deletion, insertion, or substitution). We further extend this characterization to the composition of DPA with an arbitrary base mechanism, such as inference-time Gaussian noise. We formalize these in Theorem~\ref{thm:DPA_compose_base}, and discuss proofs and details in \autoref{App:Sub_Sec:DPA}.

\begin{theorem}
\label{thm:DPA_compose_base}
Let $\mathcal D_N$ denote the DPA mechanism with $N$ partitions, and let $\mathcal B$ be any inference-time mechanism on $\mathbb X$ that is $f_b$-DP with privacy profile $\delta_b(\epsilon)$ for a neighboring relation $\approx_b$. Let $\approx_R$ be the relation corresponding to upto $R$ changes in the dataset, define the combined relation as $\approx_c$ as $(\vect X_{\text{train}}, \vect x_{\text{test}}) \approx_c (\vect {X^{'}_{\text{train}}}, \vect {x^{'}}_{\text{test}}) \iff \vect X_{\text{train}} \approx_R \vect {X^{'}}_{\text{train}} \;\; \land \;\; \vect x_{\text{test}} \approx_b \vect{x^{'}}_{\text{test}}$. Then the composed mechanism $\mathcal B \circ \mathcal D_N$, defined on $\mathcal \bigcup_{i = 1}^{N} T_i \times \mathbb X$ is $f_c$-DP with privacy profile $\delta_c(\epsilon)$ for the neighboring relation $\approx_c$, where $\delta_c(\epsilon)
=
\frac{R}{N}
+
(1 - \frac{R}{N})\,\delta_b(\epsilon)$, and 
\[
f_c(\alpha) =
\begin{dcases}
\Bigl(1 - \frac{R}{N}\Bigr)\,
f_b\!\left(\frac{\alpha}{1 - \frac{R}{N}}\right),
& \text{if } \alpha \le 1 - \frac{R}{N}, \\[6pt]
0, & \text{otherwise}.
\end{dcases}
\]
\end{theorem}\todo{Add intuition here: This allows us to use RS for evasion literature and combine it efficiently with DPA to get joint certification. Talk about the separation of results.}
These instantiations demonstrate how our framework systematically reduces joint training–test robustness certification to privacy accounting of composable randomized mechanisms.

%% file: content/experiments.tex
We evaluate two settings: \emph{(i) certification against poisoning attacks} and \emph{(ii) joint training--test certification against backdoor attacks}. Our training procedures combine selected components from sub-sampling (batching), additive Gaussian noise on training examples, DP-SGD, and DPA. We instantiate our framework on image classification for CIFAR10~\cite{krizhevsky2009learning} and MNIST~\cite{lecun1998gradient}. In \emph{(i)}, we show improved adversarial accuracy for our analysis of DP-SGD compared to prior work~\cite{liu2023enhancing}. In \emph{(ii)}, we show the utility of our framework for certifying novel threat models. We provide experiment setup details in \autoref{sec:experimental_details}.

\vspace{-0.2cm}
\subsection{Poisoning Certification}
\vspace{-0.1cm}
We consider the training-perturbation threat model that allows $R$ changes in the training data, as discussed in Section~\ref{sec:design_robust_TTP}. We compare against DPA and RDP-based accounting for DP-SGD. However, the group-privacy analysis of \citet{liu2023enhancing} does not account for the increase in sensitivity with the radius $R$. We therefore also compare against a "corrected" RDP-based accounting that we derive in \autoref{App:Sec:Corrected_RDP_group}. We report results for $8$ epochs of DP-SGD and indicate the number of DPA partitions in Figures~\ref{fig:certacc_epoch_8_CIFAR},~\ref{fig:certacc_epoch_8_MNIST}. The degradation in certification with the number of epochs is shown in \autoref{app:more}. We also observe that, for the large $\frac{\sigma}{\gamma}$ values considered here, the subsampled Gaussian behaves closely to a Gaussian with effective standard deviation $\frac{\sigma}{\gamma}$; we denote this as the Gaussian approximation in the figures.
\begin{figure*}[ht]
    \centering
    \begin{minipage}{0.44\textwidth}
        \centering
        \includegraphics[width=\textwidth]{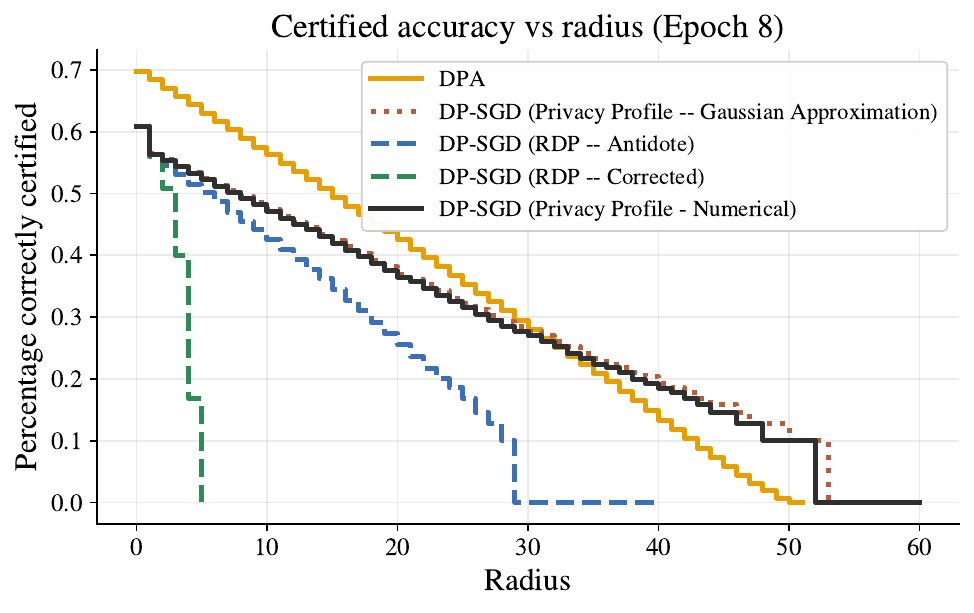}
        \caption{\small{Comparison of certified accuracies of our method (Privacy Profile -- Numerical) for DP-SGD with epoch 8, with RDP-based accounting and DPA using $100$ partitions for CIFAR-10.}}
        \label{fig:certacc_epoch_8_CIFAR}
    \end{minipage}
    \hfill
    \begin{minipage}{0.52\textwidth}
        \centering
        \includegraphics[width=\textwidth]{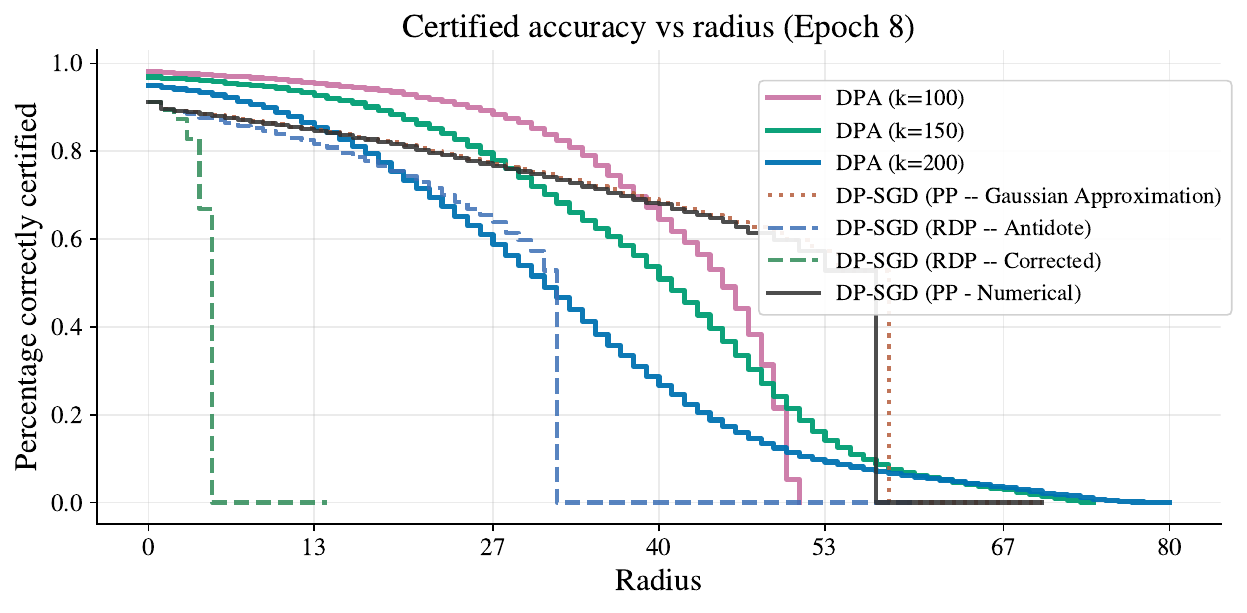}
        \caption{\small {Comparison of certified accuracies of our method (Privacy Profile -- Numerical) for DP-SGD with epoch 8, with RDP-based accounting and DPA using $100, 150, 200$ partitions for the MNIST dataset.}}
        \vspace{-0.1cm}
        \label{fig:certacc_epoch_8_MNIST}
    \end{minipage}
\end{figure*}

\vspace{-0.3cm}
\subsection{Joint Robustness Certification}
We consider two training threat models: \emph{(i)} poisoning via up to $R$ training-example additions or deletions, and \emph{(ii)} perturbations of up to $\delta_{\text{train}}$ in $\ell_2$ norm applied to at most $R$ training examples. Each is combined with an $\ell_2$ perturbation of size $\delta_{\text{test}}$ applied to the test image at inference time. To certify against \emph{(ii)}, we add noise to the training samples and use Theorem~\ref{thm:amplification_random_preprocessiing} to upper bound the privacy profile. We show results for the first threat model in Figures~\ref{fig:certified-accuracy-deltas_cifar_},~\ref{fig:certified-accuracy-deltas_mnist}, and for the second in Figure~\ref{fig:certified-accuracy-deltas}.
\begin{figure}[ht]
    \centering
    \includegraphics[width=\linewidth]{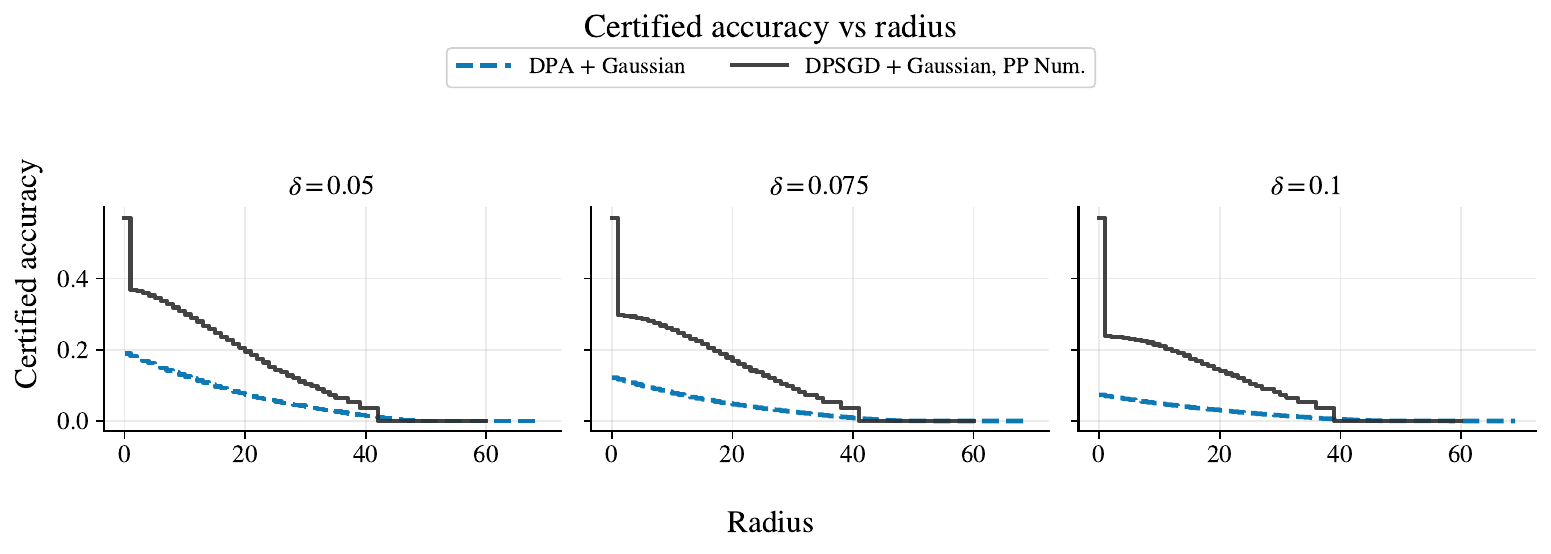}
    \caption{
    Certified accuracy vs.\ number of additions/deletions in the training dataset (\emph{Radius}) for different inference perturbations $\delta$ for the CIFAR-10 dataset. We compare DPA trained on $100$ partitions with $10$ epochs of DPSGD.}
    \label{fig:certified-accuracy-deltas_cifar_}
\end{figure}
\begin{figure}[ht]
    \centering
    \includegraphics[width=0.65\linewidth]{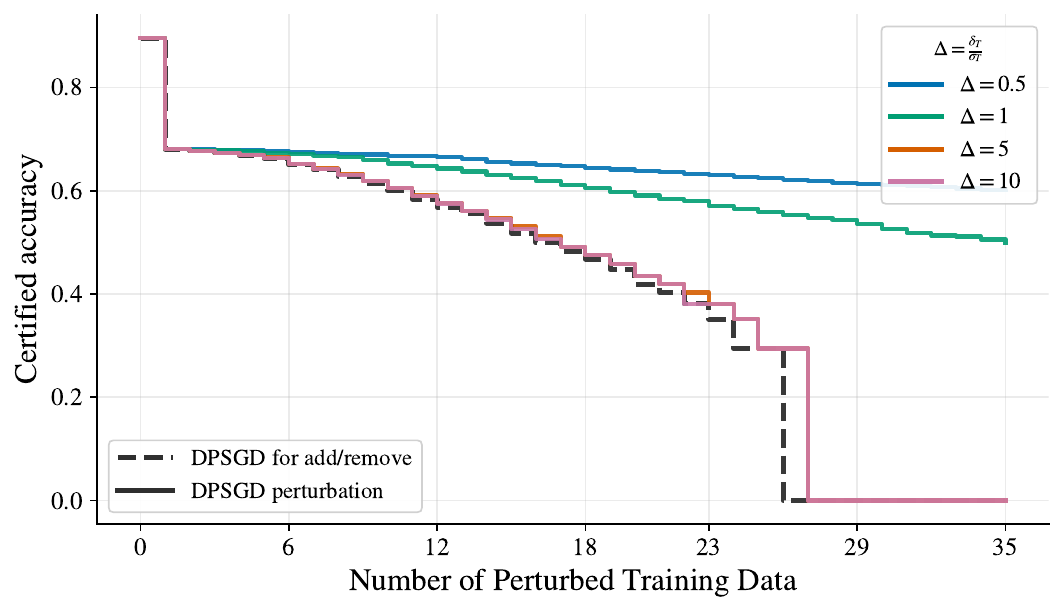}
    \caption{
Certified accuracy vs.\ the number of training examples that can be perturbed up to $\ell_2$ magnitude $\delta$, together with an inference perturbation of size $\delta_{\text{test}} = 0.2$. The plots are for MNIST models trained with DP-SGD and input Gaussian noise $\sigma = 0.2$ for $8$ epochs. Each plot corresponds to a single choice of $\delta/\sigma$.}
    \label{fig:certified-accuracy-deltas}
\end{figure}%
\begin{figure}[ht]
    \centering
    \includegraphics[width=0.85\linewidth]{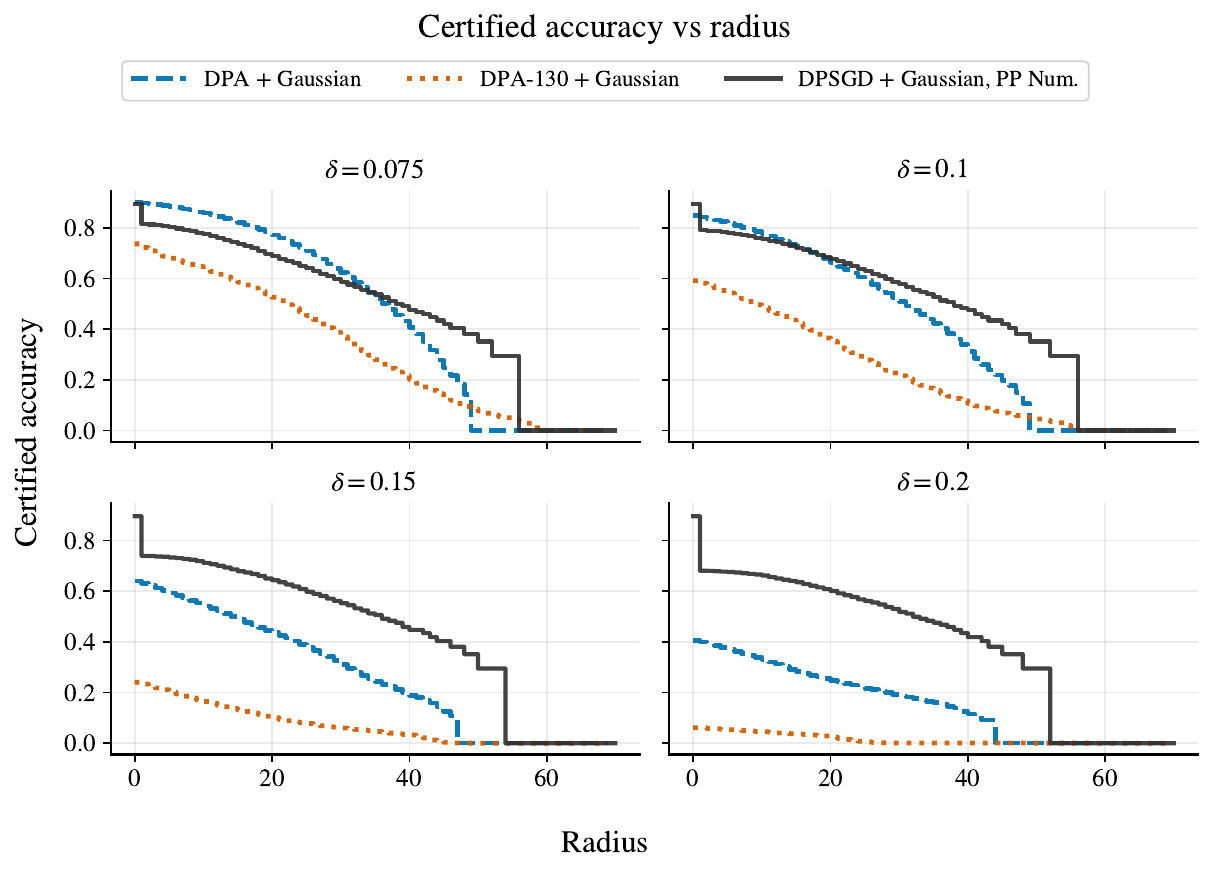}
    \caption{
    Certified accuracy vs.\ number of additions/deletions in the training dataset (\emph{Radius}) for different inference perturbations $\delta$ for the MNIST dataset. We compare DPA trained on $100$, and $130$ partitions with $8$ epochs of DPSGD.}
    \label{fig:certified-accuracy-deltas_mnist}
\end{figure}

\vspace{-0.1cm}

%% file: content/Appendix.tex
\section{Broader impacts}\label{sec:broader_impact}
This work advances theoretical foundations for certifying robustness against backdoor attacks in machine learning systems. By establishing provable guarantees against joint training-time poisoning and inference-time perturbations, our framework enables safer deployment of ML models in security-critical applications. While provable robustness may reduce model accuracy or increase computational costs, we believe these trade-offs can be mitigated long-term. Overall, our work contributes to the development of more secure and trustworthy AI systems in the long-term.

\section{Experimental details}\label{sec:experimental_details}
\paragraph{Models and Training.} For both DPA- and DP-SGD-based training, we use a ResNet18~\cite{he2016resnet} pretrained on ImageNet1K~\cite{deng2009imagenet} and available in PyTorch~\cite{paszke2019pytorch}. We used Opacus~\cite{opacus} for private training. To use the same pretrained model on MNIST, we replicate the single channel across three channels. To improve the utility of DP-SGD-trained models, we adopt the data-augmentation strategy with multiplicity $50$ proposed by \citet{de2022unlockinghighaccuracydifferentiallyprivate}. For all DP-SGD training, we use Gaussian noise with $\sigma = 3$, clipping norm $1$, and batch size $128$, and correspondingly to sub-sampling rates ($\gamma$) of $\frac{128}{50000}$ for CIFAR-10 and $\frac{128}{60000}$ for MNIST. We do not replace the BatchNorm layer with alternatives commonly used in the differential privacy literature, such as GroupNorm. Instead, we freeze the BatchNorm layers so that normalization uses the statistics obtained during pre-training. This approach correctly accounts for privacy while preserving utility. In our experiments, it significantly outperforms GroupNorm in terms of model accuracy. All models are trained on a single NVIDIA A100 GPU.
We compute a probabilistic lower/upper bound on class probabilities using Clopper-Pearson intervals for Binomials as done in \citet{cohen2019randomized}.

\textbf{DPA training details.} We train ResNet18 using Adam with a learning rate of \(0.01\), momentum \(0.9\), and weight decay \(5 \times 10^{-4}\). The training batch size is \(128\), while the inference batch size is \(300\). Models are trained for a maximum of \(400\) epochs with early stopping after \(100\) epochs without improvement on the training accuracy. We use a cosine learning rate scheduler and initialize models from pretrained weights. During training the individual classifiers, we add the same amount of Gaussian noise to the images that we use to smooth the classifier at inference. The training procedure itself is deterministic and follows the instructions described in \citep{levine2021dpa}.

\section{Equivalence of \texorpdfstring{$f$-DP}{f-DP} and Privacy Profiles}
\label{Appendix:sec:fdp_privacy_profile_equivalence}

In this section, we describe the equivalence between the primal ($f$-DP) and dual (privacy profile) formulations of differential privacy. We adopt a hypothesis testing viewpoint, which makes the shared geometric structure of the two formulations explicit.

\paragraph{Hypothesis Testing Viewpoint.}
Let $P$ and $Q$ be probability distributions on a space $\mathcal Y$. A (possibly randomized) hypothesis test
\[
H : \mathcal Y \to \{0,1\}
\]
aims to distinguish whether a sample $y \in \mathcal Y$ is drawn from $P$ or $Q$. We interpret $H(y)=1$ as deciding $Q$, and $H(y)=0$ as deciding $P$.

The performance of a test $H$ is characterized by its \emph{Type~I} and \emph{Type~II} errors:
\[
\alpha(H) := \mathbb P_{y \sim P}[H(y)=1],
\qquad
\beta(H) := \mathbb P_{y \sim Q}[H(y)=0].
\]

\paragraph{The Error Region.}
As $H$ ranges over all hypothesis tests, the set of achievable error pairs
\[
\mathcal E(P,Q)
\;:=\;
\bigl\{(\alpha(H),\beta(H)) : H \text{ is a hypothesis test}\bigr\}
\subseteq [0,1]^2
\]
forms a convex subset of the unit square. This follows because any convex combination of Type~I and Type~II error pairs can be realized by a corresponding convex combination of hypothesis tests, which is itself a valid hypothesis test.

\paragraph{Two Equivalent Descriptions.}
The region $\mathcal E(P,Q)$ admits two equivalent descriptions.

\begin{itemize}
    \item \emph{Primal (boundary) description.}  
    The lower boundary of $\mathcal E(P,Q)$ gives the smallest achievable Type~II error for each Type~I error. This defines the \emph{tradeoff function}
    \[
    \Lambda(P \,\|\, Q)(\alpha)
    \;=\;
    \inf_{H:\,\alpha(H)=\alpha} \beta(H),
    \]
    which underlies $f$-differential privacy.

    \item \emph{Dual (supporting hyperplane) description.}  
Alternatively, $\mathcal E(P,Q)$ can be characterized via the supporting hyperplanes of its lower boundary. In particular, for a slope $-e^{\epsilon}$, the supporting hyperplane has intercept
\[
1 - \delta(\epsilon) = - f^*(-e^{\epsilon}),
\]
where $f^*$ denotes the convex conjugate of the tradeoff function $f$. As shown in \cite{dong2019gaussian}, this characterization coincides with the privacy profile when the neighboring relation is \emph{symmetric}. 
\end{itemize}

In this work, we relax this assumption and extend this perspective to general \emph{asymmetric} neighboring relations. We first define the notion of optimal tradeoff function for a given mechanism $\mathcal M$ with respect to a neighboring relation $\sim$. Following, we show that it's an equivalent description of privacy to the privacy profile of $\mathcal M $ for the opposite relation $\sim_{op}$ defined by $X \sim_{\mathrm{op}} X'$ if and only if $X' \sim X$. Alternatively, the privacy profile of the mechanism $\mathcal M$ for $\sim$ is related to $f^{-1}$, via convex conjugation. Here $f^{-1}$ is left-continuous inverse of a non-increasing function $f$,

\begin{equation}
\label{eq:left_continuou_inverse}
    f^{-1}(\alpha)
:=
\inf\bigl\{ t \in [0,1] : f(t) \le \alpha \bigr\},
\end{equation}

Finally, we show that if the underlying neighboring relation $\sim$ is symmetric, the conversion between $f-$DP and privacy profile reduces to the one presented in the \citet{dong2019gaussian}. We restate this equivalence as Corollary~\ref {crl:symmetric_tradeoff_function}.

\paragraph{Optimal Tradeoff Function:} Let $\mathcal F$ be the space of tradeoff functions, i.e., by Proposition $2.2$ in \cite{dong2019gaussian}, 
\[
\mathcal F := \{f : [0,1] \to [0,1] \; \vert \; f \text{ is convex, continuous, non-increasing, and  } f(x) \leq 1 - x \}
\]
\begin{definition}(Optimal Tradeoff Function)
    \label{def:optimal_tradeoff_function_appendix}
    Given a mechanism $\mathcal M$, and neighboring relation $\sim$, we say that $f \in \mathcal F$ is the optimal tradeoff function if $\mathcal M$ is $f-$DP w.r.t., $\sim$, and for any $f{'} \in \mathcal F$ if $\mathcal M$ is $f^{'}-$DP w.r.t., $\sim$, then $f^{'} \leq f$.    
\end{definition}
 We characterize the optimal tradeoff function via Theorem \ref{th:optimal_tradeoff_function} and finally establish its connection to privacy profile in Theorem \ref{thm:primal_dual_asymmetric}. We first state the supporting lemmas needed to prove the main results, then state and prove the main results, and finally prove the supporting lemmas.

\subsection{Auxiliary Lemmas}
\paragraph{Convention:} For all the results that follow, whenever we write $f^*$, or $f^{**}$, it is implicitly implied that the convex conjugation is applied to the $\infty-$extension of the function $f$.
\begin{lemma}
\label{lemma:epigraph_monotonicity}
Let $\mathbb{X} \subseteq \mathbb{R}$ and $f : \mathbb{X} \to \mathbb{R}$. Define the property $(\mathcal P_\mathbb{X})$ for a set $A \subseteq \mathbb{X} \times \mathbb{R}$ by
\begin{equation}
\label{eq:rightinclusion}
(\mathcal P_\mathbb{X}) \qquad 
\forall (x,t) \in A,\ \forall \hat{x} \in \mathbb{X} \text{ with } \hat{x} \ge x:\ (\hat{x}, t) \in A.
\end{equation}
Then $f$ is monotonically non-increasing if and only if $\operatorname{epi}(f)$ satisfies $(\mathcal P_\mathbb{X})$.
\end{lemma}

\begin{lemma}
\label{lemma:convex_closure_preserves_right_inclusion}
Let $\mathbb{X} \subseteq \mathbb{R}$ be a closed interval and let $f : \mathbb{X} \to \mathbb{R}$. If $\operatorname{epi}(f)$ satisfies $(\mathcal P_\mathbb{X})$, then $\operatorname{cl}(\operatorname{conv}(\operatorname{epi}(f)))$ also satisfies $(\mathcal P_\mathbb{X})$.
\end{lemma}

\begin{lemma}
    \label{lemma:convex_closure_preserves_monotonicity}
    Let $\mathbb{X}\subseteq\mathbb{R}$ be a closed interval and let $f:\mathbb{X}\to\mathbb{R}$ be monotonically non-increasing.
    Consider the biconjugate of the $+\infty$-extension of $f$ to $\mathbb{R}$, and denote its restriction to $\mathbb{X}$ by $f^{**}\!\mid_\mathbb{X}$.
    Then $f^{**}\!\mid_\mathbb{X} : \mathbb{X}\to\mathbb{R}$ is monotonically non-increasing.
\end{lemma}

\begin{lemma}
    \label{lemma:HS-Div-via-HT}
    Let $P$ and $Q$ be distributions on $Y$, and $\mathcal H$ be the space of randomized hypothesis tests on $Y$. Then the HS divergence between $P$ and $Q$ can be written as the following supremum over $\mathcal H$.

    \begin{equation*}
        D_{\alpha}^{\text{HS}} = \sup_{\phi \in \mathcal H} \bigl(\mathbb E_P(\phi)-e^\varepsilon \mathbb E_Q(\phi)\bigr) 
    \end{equation*}
\end{lemma}

\begin{lemma}
    \label{lemma:order_preserved_by_inverse}
    Let $f$, $g$ be two tradeoff functions, i.e., $f,g \in \mathcal F$, such that $f \leq g$. Then the left continuous inverse of a non-increasing function preserves this order, i.e., 
    \begin{equation*}
        f^{-1} \leq g^{-1}
    \end{equation*}
\end{lemma}

\begin{lemma}
    \label{lemma:conjugation_preserves_global_lower_bound}
    Let $f: \mathbb R \to \overline {\mathbb R}$  be a proper function with a global lower bound $L$, i.e., $-\infty < L \leq f$. Then, its biconjugate also has $L$ as a global lower bound, i.e., $L \leq f^{**}$. 
\end{lemma}
 \subsection{Optimal Tradeoff Function And Primal-Dual Conversion}

\begin{theorem}
    \label{th:optimal_tradeoff_function}
    The optimal tradeoff function for a mechanism $\mathcal M$ w.r.t. a relation $\sim$ can be expressed as the convex bi-conjugate
    of the $+\infty-$extension of the pointwise infimum of the tradeoff functions over all neighboring pairs, restricted to $[0,1]$.
    \begin{equation}
        \label{eq:optimalTF}
        f = \left(\inf_{X \sim X^{'}} \Lambda(\mathcal M(X)\,\|\, \mathcal M(X^{'}))\right)^{**}\bigg|_{[0,1]}
    \end{equation}
\end{theorem}
\begin{proof}
    We first show that $f$ in \eqref{eq:optimalTF} is a tradeoff function. Define $\hat f$ as,\[\hat f:= \inf_{X \sim X'} \Lambda(\mathcal M(X)\,\|\, \mathcal M(X'))\]. This implies $f = \hat f^{**}|_{[0,1]}$. Since $[0,1]$ is convex and closed, $\operatorname{cl}(\operatorname{conv}(\operatorname{epi}(\hat f))) \subseteq X \times \mathbb R$, hence $\hat f^{**}(x) = +\infty$ for $x \notin X$. From now on, whenever we write $\hat{f}^{**}$, it is inherently implied that its domain is restricted to where it is finite, i.e., $[0,1]$. 
    
     (Convex) Holds trivially, as $f$ is a convex conjugate of $\hat f^*$. 
    
     (Monotonically non-increasing) For every pair $X \sim X^{'}$, the tradeoff function $\Lambda(\mathcal M(X) \,\|\, \mathcal M(X^{'}))$ is monotonically non-increasing and taking $\inf$ preserves that, making $\hat f$ monotonically non-increasing. Thus by Lemma \ref{lemma:convex_closure_preserves_monotonicity}, $f = \hat f^{**}$ is also monotonically non-increasing.
     
     (Less than $1 - \alpha$) The tradeoff function corresponding to every pair $X \sim X^{'}$ has an upper bound $1 - \alpha$. Therefore, $f = \hat f^{**} \leq \hat f \leq \Lambda(\mathcal M(X) \,\|\, \mathcal M(X^{'})) \leq 1 - \alpha$.

     (Greater than $0$) Since for every neighboring pair the tradeoff function is non-negative, it implies that $\hat f \geq 0$. Therefore, using Lemma \ref{lemma:conjugation_preserves_global_lower_bound}, $f = \hat f^{**} \geq 0$.
          
     (Continuous) Since $f$ is convex, it is continuous in the interior $(0,1)$ and at $1$, since $0 \leq f(\alpha) \leq 1 - \alpha$. Moreover, since it's the bi-conjugate of $\hat f$, it is lower semi-continuous. Further, $f$ is bounded above and convex, therefore it is upper semi-continuous at $0$, implying continuity at $0$. 

     Thus, $f$ is a valid tradeoff function. Now we show that it is also optimal in the sense of Definition \ref{def:optimal_tradeoff_function}.
     
     Since $f \leq \hat f \leq \Lambda(\mathcal M(X) \,\|\, \mathcal M(X^{'})$, for all neighboring pairs $X\sim X^{'}$, $\mathcal M$ is $f-$DP w.r.t. the relation $\sim$. Assume that there exist a tradeoff function $f^{'} \in \mathcal F$, such that $\mathcal M$ is $f^{'}-$DP w.r.t. $\sim$. Clearly, $f^{'} \leq \Lambda(\mathcal M(X) \,\|\, \mathcal M(X^{'})$, implying that $f^{'} \leq \hat f$, since $\hat f$ is the point-wise $\inf$. Therefore, $f^{'} \leq f$, since $f$ is the largest convex lower semi-continuous function smaller than $\hat f$. Thus showing that $f$ is the optimal tradeoff function for the mechanism $\mathcal M$, w.r.t., the neighboring relation $\sim$.
\end{proof}

\begin{lemma}
\label{lemma:relating_f_to_f_inv}
    Given a mechanism $\mathcal M$, a neighboring relation $\sim$, and its opposite relation $\sim^{op}$, defined as $X \sim X^{'}$ if and only if $X^{'} \sim^{op} X$. Let $f$ and $f^{op}$ be the optimal tradeoff functions corresponding to the relations $\sim $ and $\sim^{op}$, respectively. Then $f^{op}$ is the left continuous inverse of the non-increasing function $f$, i.e.,  
    \begin{equation*}
        f^{op} = f^{-1}
    \end{equation*}
\end{lemma}
\begin{proof}
    Since $f^{op}$ is a tradeoff function, there exist some distributions $P$ and $Q$ on some space $\mathbb Y$, such that $f^{op} = \Lambda (P \,\|\,Q)$. For all $X \sim^{op} X^{'}$, $\Lambda (P \,\|\,Q) \leq \Lambda (\mathcal M(X) \,\|\,\mathcal M(X^{'}))$ with distributions $\mathcal M(X)$, and $\mathcal M (X^{'})$ defined on $\mathbb Z$. Using Theorem $2.10$ in \citet{dong2019gaussian}, there exist a randomized map $\operatorname{\mathcal A}: \mathbb Y \to \mathbb Z$, such that $\mathcal M(X) = \mathcal A (P)$, and $\mathcal M(X^{'}) = \mathcal A (Q)$.

    We now write $(f^{op})^{-1} = \Lambda(Q,P) \leq \Lambda(\mathcal M (X^{'}) \,\|\,\mathcal M(X))$, due to post-processing inequality. Since this holds for all $X^{'} \sim X$, we get $(f^{op})^{-1} \leq f$. Using the same set of arguments starting with $f$, leads to $f^{-1} \leq f^{op}$. Further, using Lemma \ref{lemma:order_preserved_by_inverse} gives $f \leq (f^{op})^{-1}$. Thus showing $f = (f^{op})^{-1}$, or $f^{op} = f^{-1}$.
\end{proof}

\begin{theorem}
\label{thm:primal_dual_asymmetric_appendix}
Let $\mathcal M$ be a randomized mechanism with the optimal tradeoff function $f$, and privacy profile $\delta$, for the neighboring relation $\sim$. The following (equivalent) conversions hold,
\begin{itemize}
    \item (Primal to Dual) \begin{equation*}
        \delta(\epsilon) = 1 + (f^{-1})^*(-e^{\epsilon})
    \end{equation*}
    \item (Dual to Primal)\begin{align*}
        f(\alpha) = \sup_{\epsilon \in \mathbb R} e^{-\epsilon}(1 - \delta(\epsilon) -\alpha) 
    \end{align*}
\end{itemize}
\end{theorem}

\begin{proof}
\begin{itemize}
    \item (Primal to Dual) Let $X \sim X^{'}$, denote $P := \mathcal M(X)$, and $Q := \mathcal M(X^{'})$. Further let $\mathcal H$ be the space of randomized hypothesis tests, i.e.,  $\phi  \in \mathcal H$ mapping to $[0,1]$
    
    \begin{align*}
    \delta(\varepsilon) &=\sup_{X \sim X^{'}} \sup_{A} \bigl(P(A)-e^\varepsilon Q(A)\bigr)  \\
    &= \sup_{X \sim X^{'}} \sup_{\phi \in \mathcal H} \bigl(\mathbb E_P(\phi)-e^\varepsilon \mathbb E_Q(\phi)\bigr) \qquad \text{(using Lemma \ref{lemma:HS-Div-via-HT} )}\\
    &= \sup_{X \sim X^{'}}\sup_{\alpha\in[0,1]}
    \left(
    \sup_{\mathbb E_Q[\phi]=\alpha}\mathbb E_P[\phi]
    - e^\varepsilon \alpha
    \right) \\
    &= 1 + \sup_{\alpha\in[0,1]}
    \left(
     -\inf_{X \sim X^{'}} \Lambda(Q \vert \vert P)(\alpha)
    - e^\varepsilon \alpha
    \right) \\
    &= 1 + \sup_{\alpha\in[0,1]}
    \left(- e^\varepsilon \alpha
     -\inf_{X' \sim^{op} X} \Lambda(Q \vert \vert P)(\alpha) \right) \\
    &= 1 + (\inf_{X' \sim^{op} X} \Lambda(Q \vert \vert P))^*(-e^\varepsilon) \\
    &= 1 + ((\inf_{X' \sim^{op} X} \Lambda(Q \vert \vert P))^{**})^*(-e^\varepsilon) \\
    &= 1 + (f^{op})^*(-e^\varepsilon) \\
    &= 1 + (f^{-1})^*(-e^\varepsilon) \qquad\text{(using Lemma \ref{lemma:relating_f_to_f_inv})}
    \end{align*}

    which proves the claim.
    
\item (Dual to Primal)
Write $f$ as inverse of $f^{-1}$, for $\alpha > 0$
\todo{Replace this with the fact that taking the conjugate. Since f is MNI, sup needs to be taken only on negative slopes. State this as a Lemma. This idea from an invertibility perspective shows the existence of dominating pairs; derive that for a general case.}
\begin{align*}
    f(\alpha) &= \inf_{y \in [0,1]} y \qquad \text{s.t.} \qquad f^{-1}(y) - \alpha \leq 0 \\
    &= \sup_{\lambda \geq 0} \inf_{y \in [0,1]} y + \lambda f^{-1}(y) - \lambda \alpha \qquad \text{(for $\alpha > 0$, $f^{-1}(1) - \alpha < 0$; strong duality holds)}\\
    &= \sup_{\lambda > 0} \inf_{y \in [0,1]} y + \lambda f^{-1}(y) - \lambda \alpha \qquad \text{(since for $\lambda = 0$; $\inf_{y \in [0,1]} y = 0$)} \\
    &= \sup_{\lambda > 0} - \lambda(\sup_{y  \in [0,1]} y(\frac{-1}{\lambda}) - f^{-1}(y)) - \lambda \alpha \\
    &= \sup_{\varepsilon \in \mathbb R} e^{-\varepsilon} (-(f^{-1})^{*}(-e^{\varepsilon}) - \alpha)\\
    &= \sup_{\varepsilon \in \mathbb R} e^{-\varepsilon} (1 -\delta(\varepsilon) - \alpha) \qquad \text{(Using Primal to Dual connection)}
\end{align*}

Denote $\hat f(\alpha) = \sup_{\epsilon \in \mathbb R} e^{-\varepsilon}(1 - \delta(\varepsilon) - \alpha)$. To show that the above equality holds at $\alpha = 0$, i.e.,  $f(0) =  \hat f(0)$ we use the continuity argument as follows. Firstly,

\begin{align*}
    f(0) &= f(\lim_{n \to \infty} \frac{1}{n}) \\
    &= \lim_{n \to \infty} f(\frac{1}{n}) \qquad \text{($f$ is continuous)} \\
    &= \lim_{n \to \infty} \hat f(\frac{1}{n}) \qquad \text{(for positive $\alpha$, $f(\alpha) = \hat f(\alpha)$)} \\
    &\geq \lim_{n \to \infty} e^{-\varepsilon} (1 -\delta(\varepsilon) - \frac{1}{n}) \qquad \text{(for all $\varepsilon \in \mathbb R$)} \\
    &= e^{-\varepsilon} (1 -\delta(\varepsilon)) \qquad \text{(for all $\varepsilon \in \mathbb R$)} \\ 
\end{align*}
Taking $\sup$ on both sides gives,
\[
f(0) \geq \hat f(0)
\]

Now, to show the other side of the inequality, observe that $\hat f$ is non-increasing, since $\sup$ preserves inequalities. Therefore, $\hat f(\frac{1}{n}) \leq \hat f(0)$, and
\[
f(0) = \lim_{n \to\infty} \hat f(\frac{1}{n}) \leq \lim_{n \to \infty} \hat f(0) = \hat f(0)
\]

This finally proves the dual-to-primal conversion.

\end{itemize}
\end{proof}
\todo{Add a Dominating Pair discussion for asymmetric relations (or rather general relations)}
\paragraph{Asymmetric to Symmetric Case} The primal-dual relationship presented in Theorem \ref{thm:primal_dual_asymmetric} holds even for the symmetric case. It's immediately evident in the primal-to-dual conversion as $f^{-1} = f$, if the underlying relation $\sim$ is symmetric. For the dual-to-primal case, we relate the privacy profiles of the opposite relations. 

\begin{lemma}
    \label{lemma:privacy_opposite}
    Let $\sim_{op}$ be defined as, $X \sim_{op} X^{'} \; \iff \; X^{'} \sim X$, and $\delta^{op}(\epsilon), \delta(\epsilon)$, be the the privacy profiles corresponding to $\sim_{op}, \sim$, respectively.
    \[
    \delta^{op}(\epsilon) = 1 - e^{\epsilon}(1 - \delta(-\epsilon))
    \]
\end{lemma}
\begin{proof}
    For every pair of distributions denoted as $P = \mathcal M(X), Q = \mathcal M(X^{'})$, define the region 
    \[
    \mathcal R_{P,Q}(\epsilon):= \{y: \frac{dP}{dQ}(y) > e^{\epsilon}\}
    \]

    \begin{align*}
        \delta^{op}(\varepsilon) = & \sup_{X \sim^{op } X{'}} P[\mathcal R_{P,Q}(\varepsilon)] - e^{\varepsilon} Q[\mathcal R_{P,Q}(\varepsilon)] \\
        =& \sup_{X \sim X^{'}} P[\mathcal R_{Q,P}(-\varepsilon)^c] - e^{\varepsilon} \mathcal M(X^{'})[\mathcal R_{Q,P}(-\varepsilon)^c] \\
        =& e^{\varepsilon}(\sup_{X \sim X^{'}} Q[\mathcal R_{Q,P}(-\varepsilon)]  - e^{-\varepsilon}P[\mathcal R_{Q,P}(-\varepsilon)] +e^{-\varepsilon}- 1)\\
        =& 1 - e^{\epsilon}(1 - \delta(-\epsilon))
\end{align*}
\end{proof}

Finally, we restate the primal-dual relationship for a symmetric neighboring relation and show that the general result from Theorem \ref{thm:primal_dual_asymmetric} reduces to the one presented in \citet{dong2019gaussian}.

\begin{corollary}
    \label{crl:symmetric_tradeoff_function}
    Let $\mathcal M$ be a randomized mechanism with the optimal tradeoff function $f$, and privacy profile $\delta$, for the symmetric neighboring relation $\approx$. The following (equivalent) conversions hold,
\begin{itemize}
    \item (Primal to Dual) \begin{equation*}
        \delta(\epsilon) = 1 + f^*(-e^{\epsilon})
    \end{equation*}

    \item (Dual to Primal)\begin{align*}
        f(\alpha) = \sup_{\epsilon \geq 0} \max \{(1 - \delta(\epsilon) -e^{\epsilon}\alpha), e^{-\epsilon}(1 - \delta(\epsilon) -\alpha)\} 
    \end{align*} 
\end{itemize}
\end{corollary}
\begin{proof}
    To show the dual to primal conversion, use Lemma \ref{lemma:privacy_opposite} to write $f(\alpha)$ as
     \begin{align*}
        f(\alpha) = & \sup_{\epsilon \in \mathbb R} e^{-\epsilon}(1 - \delta(\epsilon) -\alpha) \\
        =& \max \{\sup_{\epsilon \geq 0}e^{-\epsilon}(1 - \delta(\epsilon) -\alpha)  , \sup_{\epsilon \leq 0}e^{-\epsilon}(1 - \delta(\epsilon) -\alpha)\} \\
        =&\max\{\sup_{\epsilon \geq 0}e^{-\epsilon}(1 - \delta(\epsilon) -\alpha), \sup_{\epsilon \geq 0}e^{\epsilon}(1 - \delta(-\epsilon) -\alpha)\} \\
        =&\max\{\sup_{\epsilon \geq 0}e^{-\epsilon}(1 - \delta(\epsilon) -\alpha), \sup_{\epsilon \geq 0} 1 - \delta(\epsilon) - e^{\epsilon}\alpha\} \qquad \text{(Lemma \ref{lemma:privacy_opposite}, and $\delta^{op} = \delta$)} \\
        =&\sup_{\epsilon \geq 0}\max\{e^{-\epsilon}(1 - \delta(\epsilon) -\alpha),  1 - \delta(\epsilon) - e^{\epsilon}\alpha\}
     \end{align*}
\end{proof}
\subsection{Auxiliary Lemmas: Proofs}
\begin{lemma}
\label{lemma:epigraph_monotonicity_proof}
Let $X \subseteq \mathbb{R}$ and $f : X \to \mathbb{R}$. Define the property $(\mathcal P_X)$ for a set $A \subseteq X \times \mathbb{R}$ by
\begin{equation}
\label{eq:rightinclusion_proofs}
(\mathcal P_X) \qquad 
\forall (x,t) \in A,\ \forall \hat{x} \in X \text{ with } \hat{x} \ge x:\ (\hat{x}, t) \in A.
\end{equation}
Then $f$ is monotonically non-increasing if and only if $\operatorname{epi}(f)$ satisfies $(\mathcal P_X)$.
\end{lemma}
\begin{proof}
Recall that the epigraph of $f$ is
\[
\operatorname{epi}(f)=\{(x,t)\in X\times\mathbb{R} : t\ge f(x)\}.
\]

($\Rightarrow$) Assume that $f$ is monotonically non-increasing on $X$. Let $x,\hat{x}\in X$ with $x\le \hat{x}$ and let $t\in\mathbb{R}$ be such that $(x,t)\in\operatorname{epi}(f)$. Then $t\ge f(x) \geq f(\hat{x})$, since $f$ is non-increasing and $x\le \hat{x}$, which implies $(\hat{x},t)\in\operatorname{epi}(f)$.

($\Leftarrow$) Conversely, assume that property \ref{eq:rightinclusion} holds. Choose any arbitrary $x,\hat{x}\in X$ with $x\le \hat{x}$. Then $(x,f(x)) \in \operatorname{epi}(f)$, therefore by property \ref{eq:rightinclusion}, it follows that $(\hat{x},f(x))\in\operatorname{epi}(f)$, i.e., $f(x)\ge f(\hat{x})$.
\end{proof}

\begin{lemma}
\label{lemma:convex_closure_preserves_right_inclusion_proof}
Let $X \subseteq \mathbb{R}$ be a closed interval and let $f : X \to \mathbb{R}$. If $\operatorname{epi}(f)$ satisfies $(\mathcal P_X)$, then $\operatorname{cl}(\operatorname{conv}(\operatorname{epi}(f)))$ also satisfies $(\mathcal P_X)$.
\end{lemma}

\begin{proof}
Let $A:=\operatorname{epi}(f)$. Since $A\subseteq X\times\mathbb{R}$ and $X$ is a closed interval, convex combinations and limits preserve the first coordinate in $X$, hence
\[
\operatorname{cl}(\operatorname{conv}(A))\subseteq X\times\mathbb{R}.
\]
We show that $(\mathcal P_X)$ is preserved first by convex hull, then by closure.

\emph{(Convex hull.)} Take $(x,t)\in \operatorname{conv}(A)$. Then
\[
(x,t)=\sum_{i=1}^n \lambda_i (x_i,t_i),
\qquad \lambda_i\ge 0,\ \sum_{i=1}^n\lambda_i=1,\ (x_i,t_i)\in A.
\]
Let $\hat x\in X$ with $\hat x \ge x$. Since $x$ is the convex combination of $\{x_i\}_{i = 1}^{n}$, there exists $j \in \{1, \cdots n\}$ such that $x_j \leq x$, consequently, $x_j + \hat x - x \in [x_j, \hat x] \subseteq X$. Define $\hat x_i = x_i$ if $i \neq j$, and $\hat x_j = x_j + \hat x - x$. We have $\hat x_i\in X$ and $\hat x_i\ge x_i$ for all $i$. Since $A$ satisfies $(\mathcal P_X)$, it follows that $(\hat x_i,t_i)\in A$. Therefore
\[
(\hat x,t)=\sum_{i=1}^n \lambda_i (\hat x_i,t_i)\in \operatorname{conv}(A),
\]
so $\operatorname{conv}(A)$ satisfies $(\mathcal P_X)$.

\emph{(Closure.)} Now let $(x,t)\in \operatorname{cl}(\operatorname{conv}(A))$ and $\hat x\in X$ with $\hat x > x$. We assume strict inequality, because the case where $\hat x = x$ is trivial. There exist a sequence $(x_k,t_k)\in \operatorname{conv}(A)$ with $(x_k,t_k) \overset{k \to \infty}{\longrightarrow} (x,t)$. Define $\delta$ as $\delta:=\hat x-x > 0$, since $x_k \to x$, there exist $N \in \mathbb N$, such that for all $k > N$, $|x - x_k| < \frac{\delta}{2}$, or $x_k < \hat x$. This implies for all $k > N$, since $(x_k, t_k) \in \operatorname{conv}(A)$, $(\hat x, t_k) \in \operatorname{conv}(A)$. Therefore, $\lim_{k \to \infty}(\hat x, t_k) = (\hat x, t) \in \operatorname{cl}(\operatorname{conv}(A))$. Thus $\operatorname{cl}(\operatorname{conv}(\operatorname{epi}(f)))$ satisfies $(\mathcal P_X)$.
\end{proof}
\begin{lemma}
    \label{lemma:convex_closure_preserves_monotonicity_proof}
    Let $X\subseteq\mathbb{R}$ be a closed interval and let $f:X\to\mathbb{R}$ be monotonically non-increasing.
    Consider the biconjugate of the $+\infty$-extension of $f$ to $\mathbb{R}$, and denote its restriction to $X$ by $f^{**}\!\mid_X$.
    Then $f^{**}\!\mid_X : X\to\mathbb{R}$ is monotonically non-increasing.
\end{lemma}
\begin{proof}
    Since $X$ is closed and convex, $\operatorname{cl}(\operatorname{conv}(\operatorname{epi}(f))) \subseteq X \times \mathbb R$. Since $\operatorname{epi}(f^{**}) = \operatorname{cl}(\operatorname{conv}(\operatorname{epi}(f)))$, and since $f$ is non-increasing on X, implies that $\operatorname{cl}(\operatorname{conv}(\operatorname{epi}(f)))$ satisfies $\mathcal P_X$(lemma \ref{lemma:epigraph_monotonicity}), therefore $\operatorname{epi}(f^{**})$ satisfies $\mathcal P_X$ (lemma \ref{lemma:convex_closure_preserves_right_inclusion}). Again, by Lemma \ref{lemma:epigraph_monotonicity}, $f^{**}$ is monotonically non-increasing on X.
\end{proof}

\begin{lemma}
    \label{lemma:HS-Div-via-HT_proof}
    Let $P$ and $Q$ be distributions on $Y$, and $\mathcal H$ be the space of hypothesis tests on $Y$. Then the HS divergence between $P$ and $Q$ can be written as the following supremum over $\mathcal H$.

    \begin{equation*}
        D_{\alpha}^{\text{HS}} = \sup_{\phi \in \mathcal H} \bigl(\mathbb E_P(\phi)-e^\varepsilon \mathbb E_Q(\phi)\bigr) 
    \end{equation*}
\end{lemma}

\begin{proof}
    For all $\phi \in \mathcal H_b$, define $A_\phi:= \{a \vert \phi(a) = 1\}$. Denote the measure $\frac{P + Q}{2}$ as $D$, then $P \ll D$, and $Q \ll D$, hence $\frac{dP}{dD}$ and $\frac{dQ}{dD}$ exist.
    By the definition of the HS divergence,
    \[
    \delta(\varepsilon)
    =\sup_{A} \bigl(P(A)-e^\varepsilon Q(A)\bigr).
    \]
    
    \begin{equation*}
        \begin{aligned}
        \sup_{A} \bigl(P(A)-e^\varepsilon Q(A)\bigr) &= \sup_{A} \bigl(\mathbb E_P(\mathbbm 1_A)-e^\varepsilon \mathbb E_Q(\mathbbm 1_A)\bigr) \\
        & \leq \sup_{\phi \in \mathcal H} \bigl(\mathbb E_P(\phi)-e^\varepsilon \mathbb E_Q(\phi)\bigr) \\ 
        &=\sup_{\phi \in \mathcal H} \bigl(\mathbb E_D(\phi( \frac{dP}{dD}-e^\varepsilon \frac{dQ}{dD}) \bigr) \\        &= P(A_\phi)-e^\varepsilon Q(A_\phi) \\ &\text{(sup is attained by a binary HT $\phi$, based on the Likelihood ratios)} \\
        & \leq \sup_A \bigl(P(A)-e^\varepsilon Q(A)\bigr) 
        \end{aligned}
    \end{equation*}
    
\end{proof}

\begin{lemma}
    \label{lemma:order_preserved_by_inverse_proof}
    Let $f$, $g$ be two tradeoff functions, i.e., $f,g \in \mathcal F$, such that $f \leq g$. Then the left continuous inverse of a non-increasing function preserves this order, i.e., 
    \begin{equation*}
        f^{-1} \leq g^{-1}
    \end{equation*}
\end{lemma}
\begin{proof}
    \begin{align*}
        &f  \leq g \\
        \implies &\{t \,| \, g(t) \leq \alpha\} \subseteq \{t \,| \, f(t) \leq \alpha\} \qquad \text{(for all $\alpha \in [0,1]$)}\\
        \implies & \inf_{t}\{t \,| \, f(t) \leq \alpha\} \leq \inf_t\{t \,| \, g(t) \leq \alpha\} \qquad \text{(for all $\alpha \in [0,1]$)}\\
        \implies & f^{-1} \leq g^{-1}
    \end{align*}
\end{proof}

\begin{lemma}
    \label{lemma:conjugation_preserves_global_lower_bound_proo}
    Let $f: \mathbb R \to \overline {\mathbb R}$  be a proper function with a global lower bound $L$, i.e., $-\infty < L \leq f$. Then, its biconjugate also has $L$ as a global lower bound, i.e., $L \leq f^{**}$. 
\end{lemma}

\begin{proof}
    \begin{align*}
    f^{**}(x) &= \sup_{\alpha} \inf_{\beta} \alpha x - \beta \alpha + f(\beta) \\
    &\geq \inf_\beta f(\beta) \qquad \text{(choose $\alpha = 0$)} \\
    &\geq L
    \end{align*}
    
\end{proof}

\section{Certifying Robustness via Primal-Dual Characterization of DP}
\label{APP_Sec: Main_robustness_results}
\begin{theorem}[Certification via the dual formulation of privacy]
    \label{thm:main_theorem_appendix}
For the unperturbed input $\vect X$, let $c_1$ and $c_2$ denote the majority and second-majority classes, respectively, and let $p_1, p_2 \in [0,1]$ satisfy $\mathbb{E}_{\mathcal M (\vect X)}\bigl[\phi_{c_1}(z)\bigr] \;\ge\; p_1 \;>\; p_2 \;\ge\; \mathbb{E}_{z \sim \mathcal M(\vect X)}\bigl[\phi_{c_2}(z)\bigr]$. Then the classifier $g(\cdot)$ is $\rho$-robust at $\mathbf X$, i.e., for all $\vect X^{'} \approx_\rho \vect X$, $g(\vect X^{'}) = g(\vect X)$ if
\begin{equation*}
\min_{i \in \{1, ..., N\}} \bigl[\max_{\varepsilon \in \mathbb R} e^{-\varepsilon}\bigl(p_1 - \delta_i(\varepsilon)\bigr) + \max_{\varepsilon \in \mathbb R} e^{-\varepsilon}\bigl(1 - p_2 - \delta_i(\varepsilon)\bigr)\bigr]
\;>\;
1   
\end{equation*}
\end{theorem}
\begin{proof}
    We verify robustness for each neighboring relation $\approx_{\rho_i}$
    Let $f_i$ be the optimal tradeoff function corresponding to the privacy profile $\delta_i(\varepsilon)$. This result directly follows from Theorem \ref{thm:primal_dual_asymmetric}, and Proposition $2.1$ in \citet{lyu2024adaptive}. Using the latter, $g(.)$ is $\rho-$robust at $\vect X$, if 

    \begin{align*}
        &f_i(1 - p_i) > 1 - f_i(p_2) \\
        \iff &    \max_{\epsilon_1 \in \mathbb R} e^{-\epsilon_1}\bigl(p_1 - \delta_i(\epsilon_1)\bigr)
    \;>\;
    \min_{\epsilon_2 \in \mathbb R} (1 - e^{-\epsilon_2}\bigl(1 - p_2 - \delta_i(\epsilon_2)\bigr))
    \end{align*}

\end{proof}

\begin{lemma}
    \label{app:lemma:tightness}
    Let $\delta_i(\cdot)$, and $f_i$ be the privacy profile and optimal tradeoff function for the decomposed relations $\approx_{\rho_i}$, with $\delta(\cdot) = \max_i \delta_i({\cdot})$ the privacy profile of $\approx_{\rho}$. If there exists a pair of data $(\vect X, \vect X')$ such that $\delta(\varepsilon) = D_{e^\varepsilon}^{\text{HS}}(\mathcal{M}(\vect X), \mathcal M(\vect X'))$, then $\min_i f_i = f$, the dual to primal conversion of $\delta(\cdot)$  
\end{lemma}

\begin{proof}
    Denote the dual to primal function as $\mathcal G$. $f = \mathcal{G}(\delta) = \mathcal G(D_{e^{(\cdot)}}^{\text{HS}}(\mathcal{M}(\vect X) , \mathcal M(\vect X'))) = \Lambda(\mathcal{M}(\vect X) \;\|\; \mathcal M(\vect X')) \geq \min_i f_i \geq f$.
     
    Therefore, no further decomposition is needed for tighter guarantees. Specifically, in the black-box sense, this holds for the sub-sampled Gaussian with respect to addition/removal~\citep[Theorem~M.4]{schuchardt2024unifiedmechanismspecificamplificationsubsampling}.. 
\end{proof}

\subsection{DP analysis of DPA and its composition with arbitrary mechanism}
\label{App:Sub_Sec:DPA}
We split the proof of \autoref{thm:DPA_compose_base} into two parts. In \autoref{App:Thm:DPA}, we interpret DPA as a randomized mechanism and derive its $f-$DP characteristic. We then use the Privacy Loss Distribution of DPA, discussed in \autoref{App:Thm:DPA}, to derive the privacy accounting of DPA combined with an arbitrary base mechanism.

\begin{theorem}[Tradeoff function for randomized DPA partition sampling]
\label{App:Thm:DPA}
Fix an integer $N\ge 1$ and a deterministic partitioning map
\[
H:\bigcup_{i=1}^N \mathcal T_i\to (\bigcup_{i=1}^N \mathcal T_i, \mathcal C),
\qquad
H(X)=(Y_1,\dots,Y_N),
\]
where each $Y_i$ is itself a dataset, and $ \mathcal {C} $ is the sigma algebra generated by the singleton sets. The randomized mechanism $\mathcal M$ maps $X_{train}$ to the distribution over $(\bigcup_{i=1}^N \mathcal T_i, \mathcal C)$, assigning each $H(X_{train})_i$ as probability mass of $1/N$. 
Let the adjacency relation be
\[
X \sim_R X' \quad \Longleftrightarrow \quad d_{\mathrm{train}}(X,X')\le R,
\]
where $d_{\mathrm{train}}$ counts the number of modified datapoints.

Then $\mathcal M$ is $f$--DP (with respect to $\sim_R$) with tradeoff function
\[
f(\alpha)\;=\;\max\Bigl\{1-\frac{R}{N}-\alpha,\,0\Bigr\},\qquad \alpha\in[0,1].
\]
\end{theorem}

\begin{proof}
Fix $X\sim_R X'$ and write $H(X)=(Y_1,\dots,Y_N)$ and $H(X')=(Y'_1,\dots,Y'_N)$.
Define $P$ and $Q$ as $\mathcal M(X)$ and $\mathcal M(X')$.
By construction, $P$ and $Q$ are uniform measures supported on
\[
S:=\{Y_1,\dots,Y_N\},\qquad S':=\{Y'_1,\dots,Y'_N\},
\]
with $P(\{y\})=\frac{1}{N}$ for $y\in S$ and $0$ otherwise (and similarly for $Q$ with $S'$).

We derive the HS divergence assuming $r$ partitions are changed, then to get the privacy profile, we take $\sup$ over $r \leq R$. If we sample $y$ from $y \sim P$, $Q(y) = 0$, with probability $\frac{r}{N}$. Therefore, the privacy loss distribution under $P$ is
\[
\mathrm{PLD}_{P/Q}^{(r)}=\log\Bigl(\frac{dP}{dQ}\Bigr)=
\begin{cases}
0, & \text{with prob. } 1-\frac{r}{N},\\
+\infty, & \text{with prob. } \frac{r}{N}.
\end{cases}
\]

Using, 
\[
\delta^{(r)}(\varepsilon)
=\mathcal D_{e^\varepsilon}(P \Vert Q)
=\mathbb E_{Y\sim \mathrm{PLD}_{P/Q}^{(r)}}\!\bigl[\,\bigl(1-e^{\varepsilon-Y}\bigr)_+\,\bigr],
\qquad (t)_+:=\max\{t,0\}.
\]

\[
\delta^{(r)}(\varepsilon)
=\Bigl(1-\frac{r}{N}\Bigr)\bigl(1-e^{\varepsilon-0}\bigr)_+
+\frac{r}{N}\bigl(1-e^{\varepsilon-\infty}\bigr)_+.
\]
For $\varepsilon\ge 0$, $(1-e^\varepsilon)_+=0$ and $e^{\varepsilon-\infty}=0$, hence
\[
\delta(\varepsilon) = \sup_{r \leq R} \delta^{(r)}(\varepsilon) = \sup_{r \leq R} \frac{r}{N} = \frac{R}{N}.
\]

\[
f(\alpha)=\sup_{\varepsilon\ge 0}e^{-\varepsilon}\bigl(1-\delta(\varepsilon)-\alpha\bigr)
=\sup_{\varepsilon\ge 0}e^{-\varepsilon}\Bigl(1-\frac{R}{N}-\alpha\Bigr)
=\max\Bigl\{1-\frac{R}{N}-\alpha,\,0\Bigr\}.
\]

For $\varepsilon<0$, $(1-e^\varepsilon)_+=1-e^\varepsilon$ and $e^{\varepsilon-\infty}=0$, hence
\[
\delta(\varepsilon) = \sup_{r \leq R}\delta^{(r)}(\varepsilon)
=\sup_{r \leq R}\Bigl(1-\frac{r}{N}\Bigr)(1-e^\varepsilon)+\frac{r}{N}
=\sup_{r \leq R}1-\Bigl(1-\frac{r}{N}\Bigr)e^\varepsilon = 1-\Bigl(1-\frac{R}{N}\Bigr)e^\varepsilon.
\]

\[
e^{-\varepsilon}\bigl(1-\delta(\varepsilon)-\alpha\bigr)
=e^{-\varepsilon}\Bigl(\bigl(1-\tfrac{R}{N}\bigr)e^\varepsilon-\alpha\Bigr)
=\Bigl(1-\tfrac{R}{N}\Bigr)-\alpha e^{-\varepsilon}.
\]
the supremum over
$\varepsilon<0$ is attained at $\varepsilon\uparrow 0$ and equals $1-\frac{R}{N}-\alpha$. Taking the max over both cases proves the claim.
\end{proof}

\begin{theorem}[Composing DPA with another mechanism]
\label{thm:DPA_compose_base_appendix}
Let $\mathcal D_N$ denote the DPA mechanism with $N$ partitions, and let $\mathcal B$ be any inference-time mechanism on $\mathbb X$ that is $f_b$-DP with privacy profile $\delta_b(\epsilon)$ for a neighboring relation $\sim_b$. Define the combined relation as $\sim_c$ as $(\vect X, \vect x) \sim_c (\vect {X^{'}}, \vect {x^{'}}) \iff d_{\text{train}}(\vect X, \vect {X^{'}}) \;\; \land \;\; \vect x \sim_b \vect{x^{'}}$. Then the composed mechanism $\mathcal B \circ \mathcal D_N$, defined on $\mathcal \bigcup_{i = 1}^{N} T_i \times \mathbb X$ is $f_c$-DP with privacy profile $\delta_c(\epsilon)$ for the neighboring relation $\sim_c$,, where
\[
f_c(\alpha) =
\begin{dcases}
\Bigl(1 - \frac{R}{N}\Bigr)\,
f_b\!\left(\frac{\alpha}{1 - \frac{R}{N}}\right),
& \text{if } \alpha \le 1 - \frac{R}{N}, \\[6pt]
0, & \text{otherwise},
\end{dcases}
\]
and
\[
\delta_c(\epsilon)
=
\frac{R}{N}
+
(1 - \frac{R}{N})\,\delta_b(\epsilon).
\]
\end{theorem}
\begin{proof}
    Let $Z_d$, $Z_b$ be the PLD of DPA and the base mechanism, respectively. The PLD of the composed mechanism is given by $Z_c = Z_d + Z_b$. Assuming $Z_d$ and $Z_b$ to be independent, we can evaluate the CDF of the composed PLD $\Phi_c(z)$ for some $z \in \mathbb R$ as

    \begin{align*}
        \Phi_c(z) &= \mathbb P[Z_c \leq z] \\
        &= \mathbb P[Z_d + Z_b \leq z] \\
        &= \mathbb E_{z_d \in \{0, +\infty\}} [\mathbb P[Z_b \leq z - z_d]] \\
        &= (1 - \frac{R}{N}) \Phi_b(z)
    \end{align*}

    This allows us to evaluate the mass at $+\infty$ as 
    \begin{align*}
        \mathbb P[Z_c = +\infty] &= 1 - \sup_z \Phi_c(z) \\
        &= 1 - (1 - \frac{R}{N})(1 - \mathbb P[Z_b = +\infty]) \\
        &= \frac{R}{N} + (1 - \frac{R}{N}) \mathbb P[Z_b = +\infty]
    \end{align*}

    Finally, we get the privacy profile from the PLD as,

    \begin{align*}
        \delta_c(\varepsilon) &=\mathbb E_{z\sim Z_c}\!\bigl[\,\bigl(1-e^{\varepsilon-z}\bigr)_+\,\bigr] \\
        &= \mathbb P[Z_c = +\infty] +  \mathbb E_{Z_d < \infty} [\mathbb E_{Z_b < \infty}[(1-e^{\varepsilon-z_d - z_b}\bigr)_+]] \\
        &=  \frac{R}{N} + (1 - \frac{R}{N}) \mathbb P[Z_b = +\infty] + (1- \frac{R}{N})\mathbb E_{Z_b < \infty}[(1-e^{\varepsilon-z_b}\bigr)_+] \\
        &=  \frac{R}{N} + (1 - \frac{R}{N}) \mathbb (P[Z_b = +\infty] + \mathbb E_{Z_b < \infty}[(1-e^{\varepsilon-z_b}\bigr)_+]) \\
        &= \frac{R}{N} + (1 - \frac{R}{N})\delta_b(\varepsilon)
    \end{align*}

    Notice that $\delta_b(\varepsilon) \leq 1$, this implies that the composed privacy profile $\delta_c(\varepsilon) = \delta_b(\varepsilon) + (1 - \delta_b(\varepsilon))\frac{R}{N}$, is monotonically increasing with respect to $R$. Therefore, the supremum over the neighboring relations is achieved at the maximum radius of perturbation $R$ in the training data.
    
  Now, to derive the optimal tradeoff function, use the result from \ref{thm:primal_dual_asymmetric}. 
    \begin{align*}
        f_c(\alpha) &= \sup_{\epsilon \in \mathbb R} e^{-\varepsilon}(1 - \delta_c(\varepsilon) - \alpha) \\
        &= \sup_{\epsilon \in \mathbb R} e^{-\varepsilon}(1 - \frac{R}{N} -(1 - \frac{R}{N})\delta_b(\varepsilon) - \alpha) \\
        &= (1 - \frac{R}{N})\sup_{\epsilon \in \mathbb R} e^{-\varepsilon}(1 - \delta_b(\varepsilon) - \frac{\alpha}{1 - \frac{R}{N}}) 
    \end{align*}

    Clearly, for $\alpha \leq 1 - \frac{R}{N}$, $f_c(\alpha) = (1 - \frac{R}{N})f_b(\frac{\alpha}{1 - \frac{R}{N}})$, and whenever $\alpha > 1 - \frac{R}{N}$, $1 - \frac{\alpha}{1 - \frac{R}{N}} - \delta(\varepsilon) < 0$ for every $\varepsilon$. Therefore, the supremum is taken as $\varepsilon \to \infty$ and is $0$.
\end{proof}

\subsection{Limitations of the Primal Perspective}
\label{appendix_subsec:limitaions_of_f-DP}
Primal DP ($f$-DP)~\citep{dong2019gaussian} offers an expressive framework for reasoning about tight composition of randomized mechanisms, including asymptotic behavior under composition. Precisely, if $f_1 = \Lambda(P_1 \,\|\, Q_1)$ and $f_2 = \Lambda(P_2 \,\|\, Q_2)$ are tradeoff functions, then their composition corresponds to the tradeoff function $\Lambda(P_1 \times P_2 \,\|\, Q_1 \times Q_2)$. However, this characterization does not provide a practical method for computing such compositions. Obtaining an analytic expression for the tradeoff function of a composition of mechanisms is
possible only when the convolutions of the corresponding Privacy Loss
Distributions (PLDs) can be computed analytically. We make
this limitation explicit by rewriting the Neyman--Pearson Lemma in terms of
PLDs in Theorem~\ref{thm:NP-PLRV}. 

We then revisit the tight characterization of the composition of randomized mechanisms from the $f$-DP perspective~\cite{dong2019gaussian}. Finally, we argue that evaluating the tradeoff function for a composition reduces to computing the convolution of the corresponding PLDs.

\begin{theorem}[Neyman--Pearson characterization via the PLRV]
\label{thm:NP-PLRV}
Assume that $\mathcal M(X)$ and $\mathcal M(X')$ admit densities $P$ and $Q$.
Define the privacy--loss random variables
\[
L_1 := \log \frac{P(y)}{Q(y)} \quad \text{for } y \sim P,
\qquad
L_2 := \log \frac{P(y)}{Q(y)} \quad \text{for } y \sim Q .
\]
Let $F_1$ and $F_2$ denote the cumulative distribution functions of $L_1$ and $L_2$.

Then for every $\alpha \in [0,1]$, the infimum over $\phi \in \mathcal H$ in
Equation~\eqref{eq:MainOPT} is attained by a likelihood--ratio test, and the
optimal value can be written in terms of the CDFs of the privacy--loss random
variables as
\[
\mathrm{Opt}(\alpha)
=
F_2\!\left(-F_1^{-1}(1-\alpha)\right),
\]
where $F_1^{-1}$ denotes the generalized inverse
\[
F_1^{-1}(u) := \inf\{ t \in \mathbb R : F_1(t) \ge u \}.
\]
\end{theorem}
\begin{proof}
Let $P$ and $Q$ denote the densities of $\mathcal M(X)$ and $\mathcal M(X')$,
respectively.  By Lemma~\ref{lem:NP}, the infimum over $\phi \in \mathcal H$
in Equation~\eqref{eq:MainOPT} is attained by a likelihood--ratio test.
Hence, there exists a threshold $\kappa \ge 0$ such that the optimal test is
\[
\phi^\star(z)=\mathbbm{1}\!\left\{\frac{P(z)}{Q(z)} \ge \kappa\right\},
\]
where $\kappa$ is chosen so that $\mathbb E_{z\sim P}[\phi^\star(z)] = \alpha$.

Define the privacy-loss random variables
$L_1 := \log \frac{P(y)}{Q(y)}$ for $y \sim P$ and
$L_2 := \log \frac{Q(y)}{P(y)}$ for $y \sim Q$,
and let $F_1$ and $F_2$ denote their cumulative distribution functions.

Taking logarithms in the definition of $\phi^\star$ gives
$\phi^\star(y) = \mathbbm{1}\{L_1(y) \ge \tau\}$, where $\tau := \log \kappa$.
The constraint $\mathbb E_{y\sim P}[\phi^\star(y)] = \alpha$ therefore becomes
$\mathbb P_{y\sim P}[L_1(y) \ge \tau] = \alpha$, which is equivalent to
$F_1(\tau) = 1-\alpha$.
By definition of the generalized inverse, $\tau = F_1^{-1}(1-\alpha)$.

It remains to express the optimal value. By Lemma~\ref{lem:NP},
\[
\mathrm{Opt}(\alpha)
= \mathbb P_{y\sim Q}\!\left[\frac{P(y)}{Q(y)} \ge \kappa\right]
= \mathbb P_{y\sim Q}[L_2(y) \le -\tau].
\]
Substituting $\tau = F_1^{-1}(1-\alpha)$ and using the definition of the CDF $F_2$ yields
\[
\mathrm{Opt}(\alpha)
= F_2\!\left(-F_1^{-1}(1-\alpha)\right),
\]
which proves the claim.

\begin{remark}
If $\alpha \le \mathbb P[L_1 = +\infty] =  \mathbb P_{y \sim P}[Q = 0]$, the optimal test rejects only on the
set where the likelihood ratio is infinite, and therefore
$\mathrm{Opt}(\alpha)=0$.  In particular, the mass at $+\infty$ in the distribution of $L_1$ determines the smallest non-trivial value of $\alpha$, and hence must be tracked when working with the privacy loss distribution. Note that the above expression with the generalized inverse is consistent with this.
\end{remark}
\end{proof}

We now recall the tight characterization of the composition of randomized
mechanisms from the $f$-DP perspective~\cite{dong2019gaussian} and express it in terms of privacy--loss distributions. Let $\mathcal M_1$ and $\mathcal M_2$ be two randomized mechanisms with
dominating pairs $(P_1,Q_1)$ and $(P_2,Q_2)$, respectively.  It is shown in~\cite{dong2019gaussian} that the composed mechanism $\mathcal M_2 \circ \mathcal M_1$ admits the product pair
\[
(P_1 \times P_2,\; Q_1 \times Q_2)
\]
as a dominating pair, where $\times$ denotes the product (independent) distribution..

We now express the tradeoff function of the composition using
Theorem~\ref{thm:NP-PLRV}.  Let
\[
L_1 := \log \frac{P_1(y_1)}{Q_1(y_1)}, \qquad y_1 \sim P_1,
\]
and
\[
L_2 := \log \frac{P_2(y_2)}{Q_2(y_2)}, \qquad y_2 \sim P_2,
\]
denote the privacy-loss random variables of the two mechanisms, and let
$F_1$ and $F_2$ be their cumulative distribution functions.

For the composed mechanism, the privacy-loss random variable becomes
\begin{align*}
L(y_1,y_2) &=
\log \frac{P_1(y_1)P_2(y_2)}{Q_1(y_1)Q_2(y_2)}, \qquad y_1 \sim P_1, \; y_2 \sim P_2 \\ 
&= L_1(y_1) + L_2(y_2).
\end{align*}
Hence, the PLD of the composition is the distribution of the sum
$L_1 + L_2$. Applying Theorem~\ref{thm:NP-PLRV} to the product pair
$(P_1 \times P_2,\; Q_1 \times Q_2)$ yields that the tradeoff function of the composed mechanism is
\[
\mathrm{Opt}_{\mathcal M_2 \circ \mathcal M_1}(\alpha)
=
F^{(Q)}_{L_1+L_2}
\!\left(
- \left(F^{(P)}_{L_1+L_2}\right)^{-1}(1-\alpha)
\right),
\]
where $F^{(P)}_{L_1+L_2}$ and $F^{(Q)}_{L_1+L_2}$ denote the CDFs of
$L_1+L_2$ under $P_1\times P_2$ and $Q_1\times Q_2$, respectively. In particular, evaluating the tradeoff function for the composition reduces
to computing the cumulative distribution function of the sum $L_1+L_2$, that is, the convolution of the two privacy-loss distributions. Therefore, to obtain an analytic expression for the tradeoff function of a composition, it is important that the convolution of the Privacy-loss distributions can be computed explicitly. For a general mechanism,
this may be intractable, but it works for the Gaussian mechanism, as shown in \citet{dong2019gaussian, lyu2024adaptive}.

\paragraph{Example: Gaussian mechanisms.}  
Let $\mathcal M$ be a Gaussian mechanism with standard deviation $\sigma$, consider the dominating pair
\[
P = \mathcal N(\Delta, \sigma^2), \qquad Q = \mathcal N(0, \sigma^2).
\]
The log-likelihood ratio is
\[
L(y) = \log \frac{dP}{dQ}(y) = \frac{\Delta^2}{2\sigma^2} - \frac{\Delta}{\sigma^2} y,
\]
which is Gaussian under $y \sim P$. Therefore, the PLD of each mechanism is Gaussian, and the PLD of a composition, given by the sum $L_1 + L_2$ of independent Gaussian PLDs, is also Gaussian. Computing the convolution then reduces to simple addition of means and variances, which makes analytic computation of the tradeoff function straightforward, as
done in \citet{dong2019gaussian,lyu2024adaptive}.

We next outline the settings in which the primal perspective works and those in which it does not. 
\paragraph{Simple Mechanisms}
Various mechanisms, such as Gaussian, Laplacian, and Uniform, have been extensively analyzed using the Neyman–Pearson lemma for evasion robustness~\cite{cohen2019randomized, teng2020ell, yang2020randomizedsmoothingshapessizes}. While the composition of Gaussian mechanisms has been analyzed in \citet{dong2019gaussian}.
\paragraph{Subsampled Gaussian Mechanisms}
Analyzing subsampled Gaussian mechanisms using the Neyman–Pearson lemma is challenging~\citep{scholten2024hierarchicalrandomizedsmoothing}. However, we need compositions of such mechanisms to obtain poisoning and joint robustness guarantees while using DP-SGD for training.
\paragraph{Heterogeneous Composition}
In general, the primal DP characterization of heterogeneous compositions is not tractable. However, we need such compositions to combine mechanisms such as DP-SGD with Gaussian mechanisms or DPA with DP-SGD. In contrast, we show in Theorem~\ref{thm:DPA_compose_base} that DPA can be analytically composed with a base mechanism that admits a tractable primal characterization.
\section{Randomized Smoothing-based Certification}
\label{Appendix:sec:RS}
 We first introduce standard randomized smoothing through a Neyman-Pearson hypothesis testing formulation~\cite{cohen2019randomized,lee2019tight}. We then present its reinterpretation through the lens of $f$-differential privacy, following~\cite{lyu2024adaptive}. We begin by defining a smoothed classifier induced by a randomized mechanism $\mathcal M$ together with a collection of hypothesis tests $\{\phi_i\}_{i = 1}^{C}$.
\begin{definition}[$C$-Class Smoothed Classifier]
\label{def:c_class_smoothed_classifier}
Let $\mathcal M$ be a randomized mechanism mapping an input
$\vect X \in \mathbb X$ to an intermediate space $\mathbb Y$.
Let $\{\phi_i\}_{i=1}^C$ be a collection of hypothesis tests
\[
\phi_i : \mathbb Y \to (\{0,1\}),
\quad i \in \{1,\dots,C\},
\]
with $\phi_C = 1 - \sum_{i \neq C} \phi_i$.
The corresponding smoothed class probabilities are defined as
\[
p_i
\;:=\;
\mathbb E_{y \sim \mathcal M(\cdot \mid \vect X)}
\bigl[\phi_i(y)\bigr],
\quad i \in \{1,\dots,C\}.
\]
The induced smoothed classifier predicts the label
\[
\argmax_{i \in \{1,\dots,C\}} \; p_i .
\]
\end{definition}

For example, in Gaussian smoothing~\cite{salman2019provably,cohen2019randomized,zhai2020macer}, a Gaussian randomized mechanism $\mathcal G_{\sigma}: \mathbb{R}^D \to \mathbb R^D$ is applied to an input $\vect x \in \mathbb R^D$, producing a Gaussian distribution $\mathcal G_{\sigma}(\vect x):= \mathcal N(\vect x, \sigma^2 \mathbb I_D)$ over $\mathbb R^D$. A neural network is then applied as a hypothesis test, and the expectation over the Gaussian noise yields the smoothed prediction.

\paragraph{Certification as an Optimization Problem}
Certifying the robustness of smoothed classifiers at an input point $X$ can be reduced to solving optimization problems of the form
\begin{equation}
    \label{eq:MainOPT_App}
    \begin{aligned}
        \text{Opt}(\alpha) :=\inf_{\phi \in \mathcal{H}} \inf_{d(X,X') \le \delta}
        &\; \mathbb{E}_{x \sim \mathcal{M}(X')}[\phi(x)] \\
        \text{s.t.} \quad
        &\; \mathbb{E}_{x \sim \mathcal{M}(X)}[\phi(x)] = \alpha.
    \end{aligned}
\end{equation}
where $\mathcal{H}$ denotes a class of hypothesis tests and $d(\cdot,\cdot)$ measures the distance between inputs~\cite{zhang2020blackboxcertificationrandomizedsmoothing}. 

In the multiclass setting, we consider a collection of class-wise hypothesis tests. Let $p_m$ denote the expected value of the hypothesis test corresponding to the predicted (majority) class, and let $p_i$ denote the expected value of the test corresponding to any other class $i \neq m$. To certify robustness, we must compare the worst-case value of the predicted class with the worst-case value of any competing class. In particular, we compute the worst-case lower bound for the predicted class, $\mathrm{OPT}(p_m)$, and the worst-case upper bound for every other class, which can be written as $1 - \mathrm{OPT}(1 - p_i)$. Robustness is certified whenever
\[
\mathrm{OPT}(p_m) > \max_{i \ne m} \left(1 - \mathrm{OPT}(1 - p_i)\right).
\]
We now state the Neyman–Pearson lemma~\cite{neyman1933ix} and an $f$-DP–based reformulation to characterize the optimization over $\mathcal H$ in Equation~\ref{eq:MainOPT}.
\begin{lemma}[Neyman--Pearson Characterization \cite{neyman1933ix}]
\label{lem:NP}
For fix $X, X'$ and $\alpha \in [0,1]$, the infimum over $\phi \in \mathcal H$ in Equation~\eqref{eq:MainOPT} is attained by a likelihood–ratio test.  
In particular, there exists a threshold $\kappa \ge 0$ such that the optimal hypothesis test
$\phi^\star \in \mathcal H$ is given by
\[
\phi^\star(z)
\;=\;
\mathbbm 1\!\left\{
\frac{d\mathcal M(X)}{d\mathcal M(X')}(z) \ge \kappa
\right\},
\]
where $\kappa$ is chosen to satisfy
\[
\mathbb E_{z \sim \mathcal M(X)}[\phi^\star(z)] = \alpha .
\]
The corresponding optimal value is
\[
\text{Opt}(\alpha)
=
\mathbb P_{z \sim \mathcal M(X')}
\!\left[
\frac{d\mathcal M(X)}{d\mathcal M(X')}(z) \ge \kappa
\right].
\]
\end{lemma}
\todo{Add more intuition here}
\begin{lemma}[$f$-DP Reformulation of Robustness Certification \cite{lyu2024adaptive}]
\label{lemma:fdp_reformulation_App}
Let $\sim_\delta$ be the relation defined by
\[
X \sim_\delta X'
\quad \text{if and only if} \quad
d(X,X') \le \delta.
\]
If the mechanism $\mathcal M$ satisfies $f$-differential privacy with respect to $\sim_\delta$, then
\[
\text{Opt}(\alpha) \;\ge\; f(1-\alpha).
\]
\end{lemma}
\paragraph{Smooth Joint Training-Test Procedure}
We introduce the randomized smoothing framework for a joint training--test procedure. 
This formulation subsumes robustness guarantees for the training algorithm alone and 
inference-time robustness as special cases.

\begin{definition}[Randomized Joint Training--Test Procedure]
\label{def:Joint-training-procedure}
A randomized joint training--test procedure $\mathcal J$ for $C$-class classification operates as follows.
Given an input $(\vect X_{\text{train}}, \vect x_{\text{test}})
\;\in\;
\bigcup_{i=1}^{\infty} T_i \times \mathbb X$, the procedure first applies a randomized mechanism
\[
\mathcal M : \bigcup_{i=1}^{\infty} T_i \times \mathbb X \;\to\; \mathbb Y,
\]
mapping the input to an intermediate space $\mathbb Y$.
This is followed by a collection of $C$ hypothesis tests
\[
\phi_i : \mathbb Y \to \{0,1\},
\qquad i \in \{1,\dots,C\},
\]
with $\phi_C \;=\; 1 - \sum_{i \neq C} \phi_i$. For each $i \in \{1,\dots,C\}$, the corresponding expected value
\[
\mathbb E_{y \sim \mathcal M(\vect X_{\text{train}}, \vect x_{\text{test}})}
\bigl[\phi_i(y)\bigr]
\]
defines the (smoothed) probability of predicting class~$i$.
The final prediction of the procedure is given by
\[
\argmax_{i \in \{1,\dots,C\}}
\mathbb E_{y \sim \mathcal M(\vect X_{\text{train}}, \vect x_{\text{test}})}
\bigl[\phi_i(y)\bigr].
\]
\end{definition}

\paragraph{Example: DPSGD and Inference-Time Gaussian Noise}
Let $\mathcal M_i$ denote the randomized mechanism corresponding to the $i$-th iteration of DP-SGD. After $I$ iterations, the composed training-time mechanism. 
\[
\mathcal M_{\text{train}}
\;:=\;
\mathcal M_I \circ \mathcal M_{I-1} \circ \cdots \circ \mathcal M_1
\]
is applied to the training dataset $\vect X_{\text{train}}$, inducing a distribution $\mathcal M_{\text{train}}(\vect X_{\text{train}})$ over the learned network parameters $w$. At inference time, an additional Gaussian randomized mechanism $\mathcal G_\sigma(\vect x_{\text{test}}) := \mathcal N(\vect x_{\text{test}}, \sigma^2 \mathbb I)$ is applied to the test input. The overall randomized joint training--test mechanism is thus given by\[
\mathcal M(\vect X_{\text{train}}, \vect x_{\text{test}})
\;=\;
\bigl(
w \sim \mathcal M_{\text{train}}(\vect X_{\text{train}}),
\;
\tilde{\vect x} \sim \mathcal G_\sigma(\vect x_{\text{test}})
\bigr).
\]
A collection of hypothesis tests $\{\phi_i\}_{i=1}^C$, typically implemented as the components of a neural network’s softmax output, maps the pair $(w, \tilde{\vect x})$ to class-wise predictions. The smoothed probability of predicting class~$i$ is given by
\[
p_i
\;=\;
\mathbb E_{(w, \tilde{\vect x}) \sim \mathcal M(\vect X_{\text{train}}, \vect x_{\text{test}})}
\bigl[\phi_i(w, \tilde{\vect x})\bigr],
\qquad i \in \{1,\dots,C\}.
\]

\paragraph{Example: DPA and Inference-Time Noise}
DPA applies a deterministic partitioning function $H$ and maps the training dataset  \(
\vect X_{\text{train}} \in \bigcup_{i=1}^{\infty} T_i\) to a collection of $N$ disjoint subsets \(
\mathbf Y := \{Y_1,\dots,Y_N\}, %
\) where each partition $Y_i \in \bigcup_{i=1}^{\infty}T_i$ is itself a dataset. The mechanism $\mathcal M_{\text{DPA}}$ samples a partition $Y_i$ to get a distribution $\mathcal{M_{\text{DPA}}}(\vect X_{\text{train}})$, over the partitions (training datasets). Finally, a model is trained on the sampled dataset, and at inference time $\mathcal G_{\sigma}$ is applied to the test data, yielding a joint training--test mechanism \(
\mathcal M(\vect X_{\text{train}}, \vect x_{\text{test}}) = \bigl(
Y_i \sim \mathcal M_{\text{train}}(\vect X_{\text{train}}),
\;
\tilde{\vect x} \sim \mathcal G_\sigma(\vect x_{\text{test}})
\bigr) \). The process of training a model and evaluation class prediction at the noisy test sample is modeled by a collection of hypothesis tests $\{\phi_i\}_{i=1}^{C}$, leading to class probabilities as for a given $(\vect X_{\text{train}}, \vect x_{\text{test}})$ pair,
\[
 p_i = \mathbb E_{Y \sim \mathcal{M}_{\text{train}}(\vect X_{\text{train}}), \tilde{\vect x} \sim \mathcal G_\sigma(\vect x_{\text{test}})}[\phi_i(Y, \tilde{\vect x})].
\]

\section{Corrected group privacy analysis of Sub-sampled Gaussian using R\'enyi-DP}
\label{App:Sec:Corrected_RDP_group}
This section presents a corrected version of the group privacy guarantee for the
Sub-sampled Gaussian Mechanism (SGM) derived in \cite[Theorem~8]{liu2023enhancing}.
The issue is that the supporting theorem used in the original argument assumes
that the underlying function $f$ has $\ell_2$-sensitivity at most $1$ with respect
to datasets differing in up to $r$ examples. However, in the application to Theorem $8$, the assumption is that $f$ has $\ell_2$-sensitivity at most $1$ for datasets differing in one example. Consequently, for datasets differing in up to $r$ examples, the worst-case
$\ell_2$-sensitivity is at most $r$, not $1$. 
\begin{theorem}[Corrected improved R\'enyi-DP group privacy for SGM]
Let $\mathcal M$ be the Sub-sampled Gaussian Mechanism (SGM) with sampling
probability $q$ and Gaussian noise variance $\sigma^2 I_d$. Suppose that the
underlying function $f$ has single-example $\ell_2$-sensitivity at most $1$,
i.e., for any pair of datasets $D,D'$ differing in one example,
\[
\|f(D)-f(D')\|_2 \le 1 .
\]
Let $D_1$ and $D_3$ be two datasets such that
\[
D_3 \in \mathcal B(D_1,r),
\]
i.e., $D_1$ and $D_3$ differ in at most $r$ examples. Then, for any order
$\alpha>1$,
\[
D_\alpha\bigl(\mathcal M(D_1)\,\|\,\mathcal M(D_3)\bigr)
\le
\mathrm{SG}\!\left(
\alpha,1-(1-q)^r,\frac{\sigma}{r}
\right),
\]
and similarly,
\[
D_\alpha\bigl(\mathcal M(D_3)\,\|\,\mathcal M(D_1)\bigr)
\le
\mathrm{SG}\!\left(
\alpha,1-(1-q)^r,\frac{\sigma}{r}
\right).
\]

Here, for sampling probability $\bar q$ and noise level $\bar\sigma$,
$\mathrm{SG}(\alpha,\bar q,\bar\sigma)$ denotes the standard R\'enyi-DP bound
for the unit-sensitivity Sub-sampled Gaussian Mechanism:
\[
\mathrm{SG}(\alpha,\bar q,\bar\sigma)
:=
\max\left\{
D_\alpha\!\left(
\mathcal N(0,\bar\sigma^2)
\;\middle\|\;
(1-\bar q)\mathcal N(0,\bar\sigma^2)
+
\bar q\,\mathcal N(1,\bar\sigma^2)
\right),
\right.
\]
\[
\left.
D_\alpha\!\left(
(1-\bar q)\mathcal N(0,\bar\sigma^2)
+
\bar q\,\mathcal N(1,\bar\sigma^2)
\;\middle\|\;
\mathcal N(0,\bar\sigma^2)
\right)
\right\}.
\]
\end{theorem}
\begin{proof}
Let $S,S'$ be a pair of datasets that differ in $r$ examples, and write
\[
S' = S\cup\{x_1,\ldots,x_r\}.
\]
Following the argument of \cite[Theorem~10]{liu2023enhancing}, using quasi-convexity of R\'enyi divergence, we obtain
\[
\begin{aligned}
D_\alpha(\mathcal M(S)\|\mathcal M(S'))
&\le
\sup_T
D_\alpha\left(
\mathcal N(0,\sigma^2 I_d)
\;\middle\|\;
\begin{aligned}
&\sum_{V\subseteq \{x_1,\ldots,x_r\}}
q^{|V|}(1-q)^{r-|V|} \\
&\qquad\qquad\cdot
\mathcal N(f(T\cup V)-f(T),\sigma^2 I_d)
\end{aligned}
\right).
\end{aligned}
\]

\textbf{Correction:} The key correction is the following sensitivity bound. Since $f$ has single-example $\ell_2$-sensitivity at most $1$, for every
$V\subseteq\{x_1,\ldots,x_r\}$,
\[
\|f(T\cup V)-f(T)\|_2 \le |V| \le r .
\]
Using the rotational symmetry of the Gaussian noise and the additivity of
R\'enyi divergence for product distributions, the above is bounded by the
one-dimensional worst case
\[
D_\alpha(\mathcal M(S)\|\mathcal M(S'))
\le
\sup_{c\le r}
D_\alpha\left(
\mathcal N(0,\sigma^2)
\;\middle\|\;
(1-q)^r\mathcal N(0,\sigma^2)
+
\bigl(1-(1-q)^r\bigr)\mathcal N(c,\sigma^2)
\right).
\]
Equivalently,
\begin{align*}
D_\alpha(\mathcal M(S)\|\mathcal M(S'))
&\le
\sup_{c'\le 1}
D_\alpha\left(
\mathcal N(0,(\sigma/(rc'))^2)
\;\middle\|\;
\begin{aligned}
& (1-q)^r\mathcal N(0,(\sigma/(rc'))^2) \\
&\quad + \bigl(1-(1-q)^r\bigr)
\mathcal N(1,(\sigma/(rc'))^2)
\end{aligned}
\right).
\end{align*}

Follow the rest of the arguments of \cite[Theorem~10]{liu2023enhancing}, to conclude:
\[
D_\alpha(\mathcal M(S)\|\mathcal M(S'))
\le
D_\alpha\left(
\mathcal N(0,(\sigma/r)^2)
\;\middle\|\;
(1-q')\mathcal N(0,(\sigma/r)^2)
+
q'\mathcal N(1,(\sigma/r)^2)
\right),
\]
where
\[
q' = 1-(1-q)^r .
\]
The reverse direction is identical and gives
\[
D_\alpha(\mathcal M(S')\|\mathcal M(S))
\le
D_\alpha\left(
(1-q')\mathcal N(0,(\sigma/r)^2)
+
q'\mathcal N(1,(\sigma/r)^2)
\;\middle\|\;
\mathcal N(0,(\sigma/r)^2)
\right).
\]
Thus the sampled Gaussian mechanism satisfies the claimed bound with effective
sampling probability $q'=1-(1-q)^r$ and effective noise level $\sigma/r$.
\end{proof}
\clearpage

\input{content/random_preprocessing}

\clearpage

\section{Additional experiments}\label{app:more}

\begin{figure}[htbp]
    \centering
    \begin{minipage}{0.48\textwidth}
        \centering
        \includegraphics[width=\textwidth]{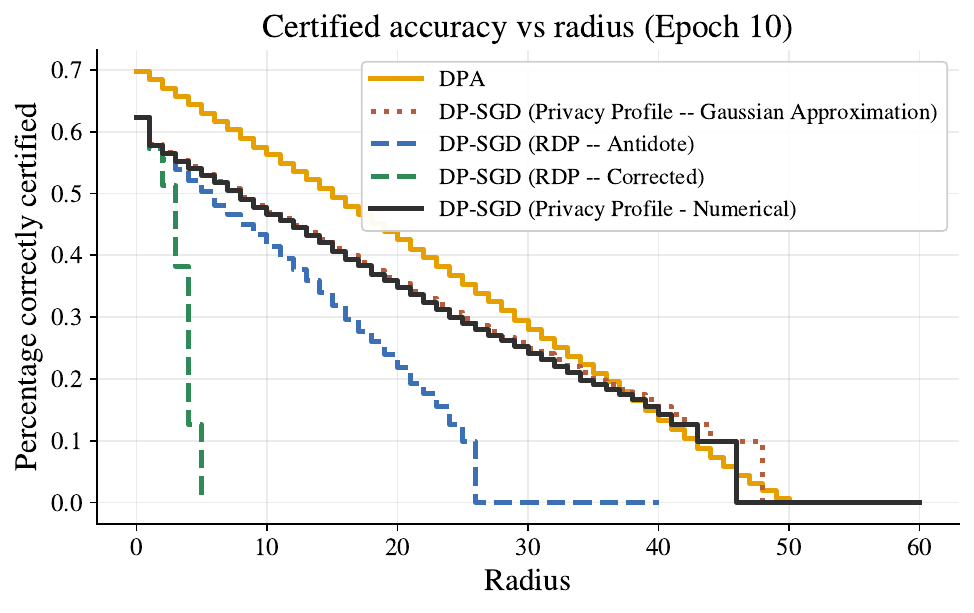}
        \caption{Comparison of certified accuracies of our method (Privacy Profile -- Numerical) for DP-SGD with epoch 10, with RDP-based accounting and DPA using $100$ partitions for CIFAR-10.}
        \label{fig:cifar10_certacc_sigma_0}
    \end{minipage}
    \hfill
    \begin{minipage}{0.48\textwidth}
        \centering
        \includegraphics[width=\textwidth]{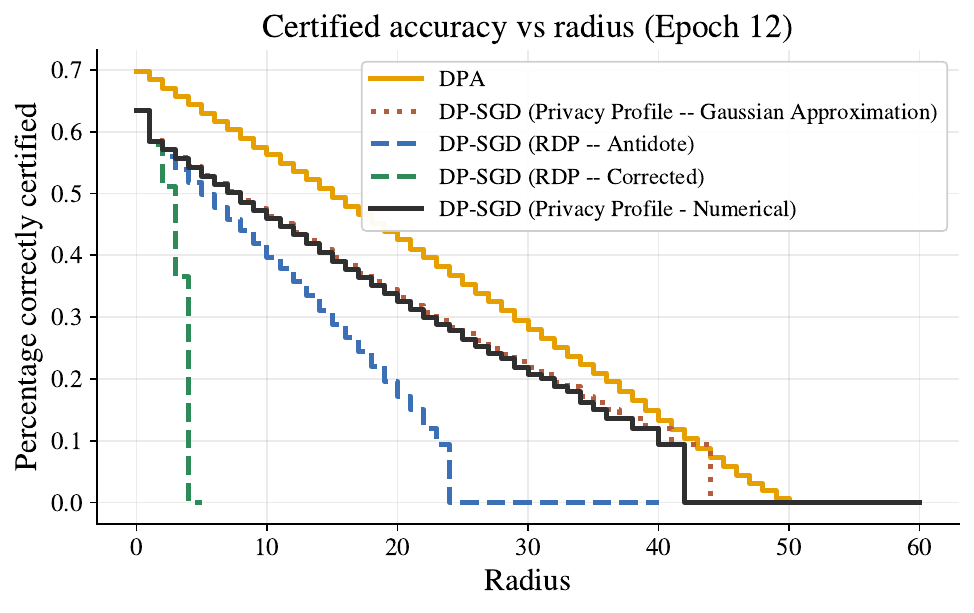}
        \caption{Comparison of certified accuracies of our method (Privacy Profile -- Numerical) for DP-SGD with epoch 12, with RDP-based accounting and DPA using $100$ partitions for CIFAR-10.}
        \label{fig:cifar10_certacc_sigma_0.01}
    \end{minipage}
\end{figure}

\begin{figure}[htbp]
    \centering
    \begin{minipage}{0.48\textwidth}
        \centering
        \includegraphics[width=\textwidth]{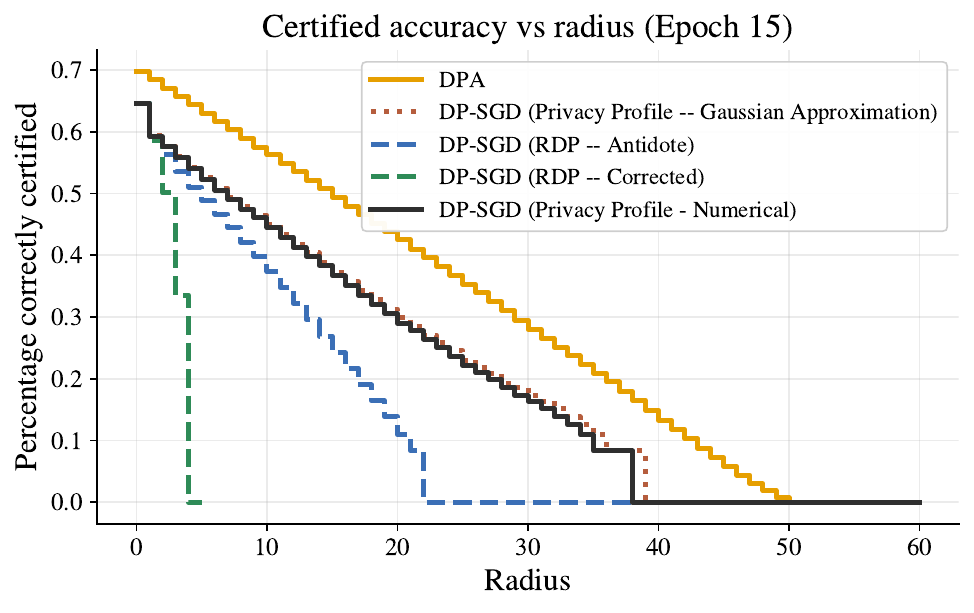}
        \caption{\small{Comparison of certified accuracies of our method (Privacy Profile -- Numerical) for DP-SGD with epoch 15, with RDP-based accounting and DPA using $100$ partitions for CIFAR-10.}}
        \label{fig:cifar10_certacc_sigma_0.02}
    \end{minipage}
    \hfill
    \begin{minipage}{0.48\textwidth}
        \centering
        \includegraphics[width=\textwidth]{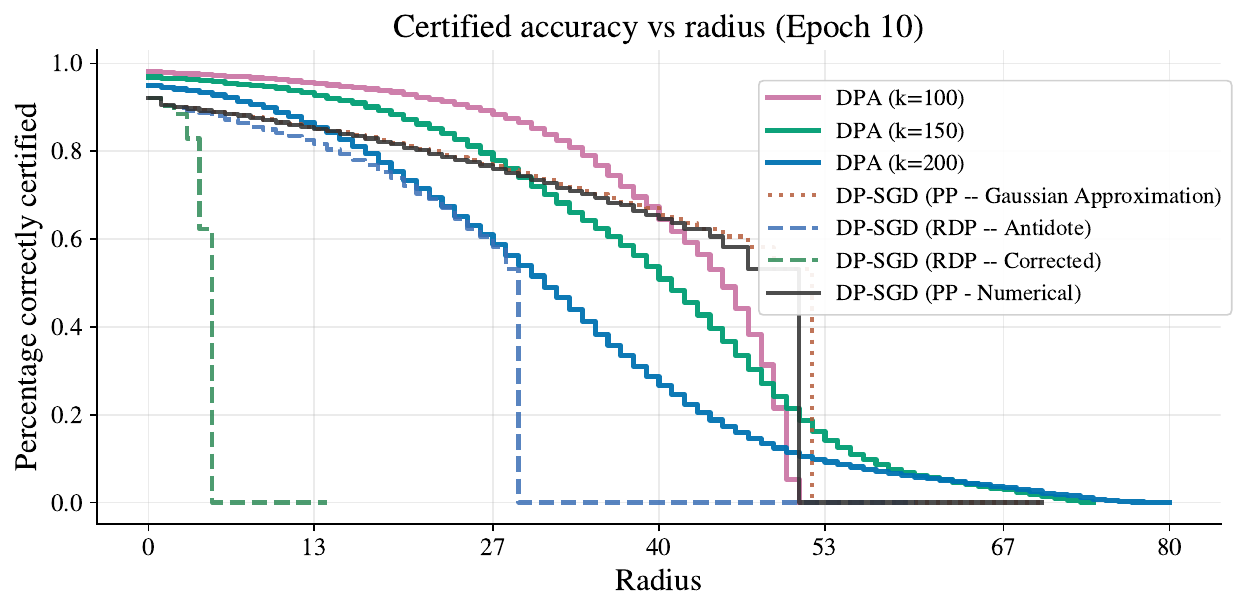}
        \caption{Comparison of certified accuracies of our method (Privacy Profile -- Numerical) for DP-SGD with epoch 10, with RDP-based accounting and DPA using $100, 150, 200$ partitions for MNIST.}
        \label{fig:cifar10_certacc_sigma_0.03}
    \end{minipage}
\end{figure}

\begin{figure}[htbp]
    \centering
    \begin{minipage}{0.48\textwidth}
        \centering
        \includegraphics[width=\textwidth]{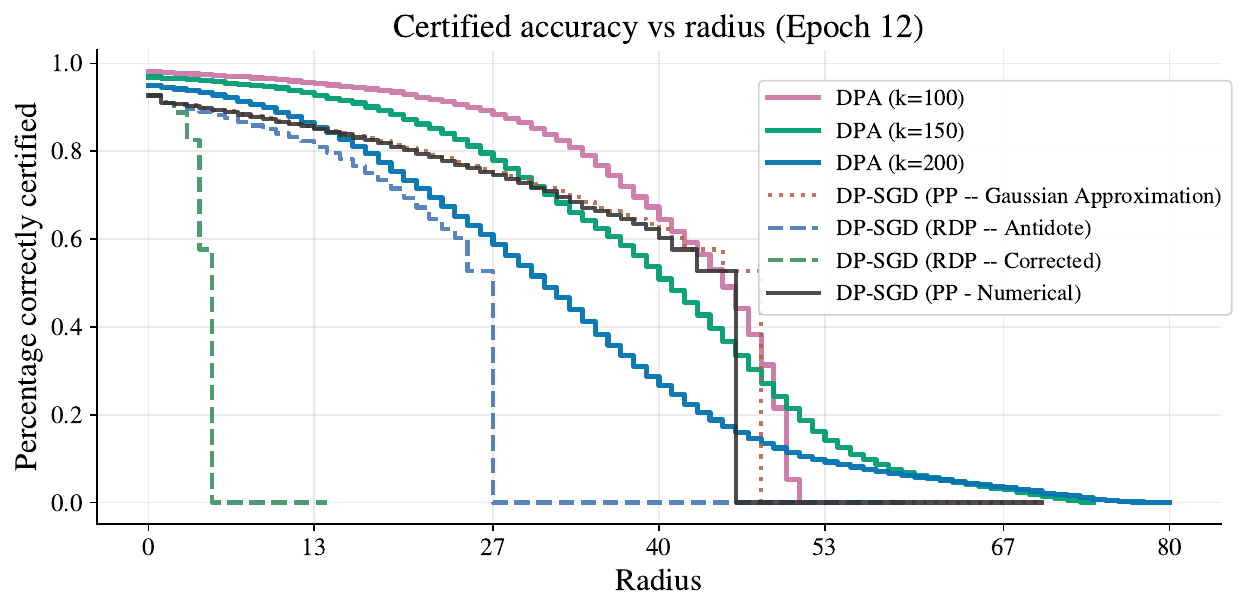}
         \caption{Comparison of certified accuracies of our method (Privacy Profile -- Numerical) for DP-SGD with epoch 12, with RDP-based accounting and DPA using $100, 150, 200$ partitions for MNIST.}
        \label{fig:cifar10_certacc_sigma_0.04}
    \end{minipage}
    \hfill
    \begin{minipage}{0.48\textwidth}
        \centering
        \includegraphics[width=\textwidth]{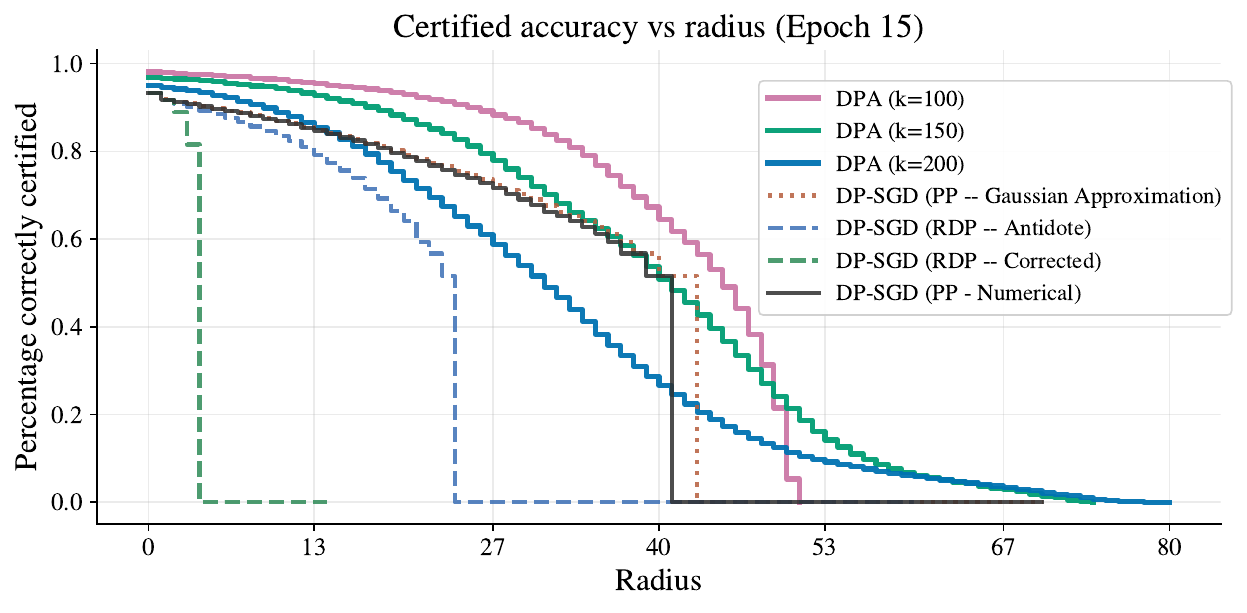}
         \caption{Comparison of certified accuracies of our method (Privacy Profile -- Numerical) for DP-SGD with epoch 15, with RDP-based accounting and DPA using $100, 150, 200$ partitions for MNIST.}
        \label{fig:cifar10_certacc_sigma_0.05}
    \end{minipage}
\end{figure}

%% file: content/random_preprocessing.tex
\section{Amplification by random preprocessing}

\begin{theorem}
\label{thm:amplification_random_preprocessiing}
    Let $X$ and $X'$ be datasets of $N$ $d$-dimensional real-valued vectors.
    Assume that they are related via substitution of $K$ records with bounded norm $r$, i.e., additive perturbation of $K$ records
    with $\ell_2$-norm $r$.
    Consider the mechanism $M(y \mid X) = \sum_{\vy \in \{0,1\}^N} \int B(y \mid \tilde{X}, \vy) G(\tilde{X} \mid X)\cdot  S(\vy) \mathrm{d} \tilde{X} $,
    where $S(\vy)$ is a distribution over Poisson subsampling indicator vectors,
    $G(\tilde{X} \mid X)$ adds additive Gaussian noise with scale $\sigma_\mathrm{in}$
    to all samples and
    base mechanism $B(y \mid \tilde{X}, \vy)$ releases noisy clipped gradients only for samples with $y = 1$, i.e., DP-SGD.
    Let $P, Q$ be a dominating pair of $B$ under insertion of $K$ records and removal of $K$ records in $\tilde{X}$.
    Then the additive noise subsampled mechanism $M$ admits the following bound on its privacy profile for all neighboring $X \simeq X'$:
    \begin{equation*}
        \mathcal{H}_\alpha(M(\cdot \mid X) || M(\cdot \mid X')) \leq \mathrm{TVD}(\mathcal{N}(0,\sigma_\mathrm{in}) || \mathcal{N}(r, \sigma_\mathrm{in})) \cdot \mathcal{H}_\alpha(P || Q),
    \end{equation*}
    where TVD is the total variation distance of univariate Gaussians with scale $\sigma_\mathrm{in}$.
\end{theorem}
\begin{proof}
    Via post-processing property of differential privacy, the mechanism is at least as private as a mechanism $M'$ that releases all records but the substituted ones without any subsampling or additive noise.
    This mechanism $M'$ is equally private as only applying our original mechanism $M$ to the $K$ substituted records, i.e.,
    we can analyze w.l.o.g. the setting where $N=K$, i.e., there are no non-substituted records in $S$ and $S'$.
    A standard result from optimal transport theory is that any pair of distributions admits a maximal coupling (see~\cite{balle2018amplification}) for a discussion in the context of differential privacy. Via projection of this maximal coupling onto its marginals~\cite{balle2018amplification}, we can show
    \begin{align*}
        &G(\tilde{X} \mid X) = (1-w) \cdot G_0 + w \cdot G_1\\
        &G(\tilde{X} \mid X') = (1-w) \cdot G_0 + w \cdot G_1' 
    \end{align*}
    with
    \begin{equation*}
        w = \mathrm{TVD}(\mathcal{N}(0,\sigma_\mathrm{in}^2 \mathbf{I}) || \mathcal{N}(r, \sigma_\mathrm{in}^2 \mathbf{I}))
        = \mathrm{TVD}(\mathcal{N}(0,\sigma_\mathrm{in}) || \mathcal{N}(r, \sigma_\mathrm{in}))
    \end{equation*}
    where the last result follows from translation-invariance of Gaussian f-divergences.

    Define mechanisms
    \begin{align*}
        M_0 = M(y \mid X) = \sum_{\vy \in \{0,1\}^N} \int B(y \mid \tilde{X}, \vy) G_0(\tilde{X})\cdot  S(\vy) \mathrm{d} \tilde{X}
        \\
        M_1 = M(y \mid X) = \sum_{\vy \in \{0,1\}^N} \int B(y \mid \tilde{X}, \vy) G_1(\tilde{X})\cdot  S(\vy) \mathrm{d} \tilde{X}
        \\
        M_1' = M(y \mid X) = \sum_{\vy \in \{0,1\}^N} \int B(y \mid \tilde{X}, \vy) G_1(\tilde{X})\cdot  S(\vy) \mathrm{d} \tilde{X} 
    \end{align*}
    The hockey-stick divergence is jointly convex in the space of distributions~\cite{balle2018amplification} and thus quasi-convex.
    It thus follows immediately that
    \begin{align*}
         &\mathcal{H}_\alpha(M(\cdot \mid X) || M(\cdot \mid X'))
         \\
         &
         w
         \leq H_\alpha(M_1(\cdot) || M'_1(\cdot))
         +
         (1-w) \cdot 
         H_\alpha(M_0(\cdot) || M'_0(\cdot))
         \\
         &
         \leq 
         w
         H_\alpha(M_1(\cdot) || M'_1(\cdot))
         \\
         &
         \leq 
         w
         H_\alpha(
         \sum_{\vy \in \{0,1\}^N}  B(y \mid \tilde{X}, \vy) \cdot  S(\vy)
         || 
         \sum_{\vy \in \{0,1\}^N}  B(y \mid \tilde{X'}, \vy) \cdot  S(\vy)
         )
    \end{align*}
    where the last step is due to quasi-convexity applied to the additive noise distribution, and $X$ and $X'$ are two arbitrary datasets with $K$ records.
    By definition, any two datasets of size $K$ are related by $K$ insertions and $K$ removals.
    The result then immediately follows from our assumption about dominance of $(P,Q)$.
\end{proof}